\documentclass[lettersize,journal]{IEEEtran}
\usepackage{amsmath,amsfonts}
\usepackage{algorithmic}
\usepackage{algorithm}
\usepackage{array}
\usepackage{booktabs}
\usepackage[caption=false,font=normalsize,labelfont=sf,textfont=sf]{subfig}
\usepackage{textcomp}
\usepackage{stfloats}
\usepackage{url}
\usepackage{verbatim}
\usepackage{graphicx}
\usepackage{cite}
\usepackage{pifont} % provides ding symbols for check/x marks
\usepackage{multirow} % provides \multirow for table cells spanning multiple rows
\usepackage{amsmath}  % 数学公式支持
\usepackage{amssymb}  % 数学符号
\usepackage{amsthm}   % 定理环境支持
\newtheorem{theorem}{Theorem}

\usepackage[table]{xcolor} % enable \rowcolor and table coloring
\newcommand{\cmark}{\ding{51}}%
\newcommand{\xmark}{\ding{55}}%
\begin{document}

\title{HeterCSI: Channel-Adaptive Heterogeneous CSI Pretraining Framework for Generalized Wireless Foundation Models}

\author{Chenyu Zhang, Xinchen Lyu, Chenshan Ren, Shuhan Liu, Qimei Cui, and Xiaofeng Tao
        % <-this % stops a space
    \thanks{This paper was supported in part by the National Natural Science Foundation of China under Grants 62371059 and 62301622.}
    \thanks{C. Zhang is with National Engineering Research Center for Mobile Network
    Technologies, Beijing University of Posts and Telecommunications, Beijing
    100876, China (e-mail: zhangchenyu2024@bupt.edu.cn).}
    \thanks{X. Lyu, Q. Cui, X. Tao are with National Engineering Research
    Center for Mobile Network Technologies, Beijing University of Posts
    and Telecommunications, Beijing 100876, China, and also with the Department of Broadband Communication, Pengcheng Laboratory, Shenzhen
    518055, China (e-mail: lvxinchen@bupt.edu.cn; 
    cuiqimei@bupt.edu.cn; taoxf@bupt.edu.cn).}
    \thanks{C. Ren is with the Key Laboratory of Ethnic Language Intelligent Analysis and Security Governance of MOE, Minzu University of China, Beijing 100081, China (e-mail: renchenshan06@163.com).}
    \thanks{S. Liu is with China Telecom Corporation Limited Gansu Branch, Gansu 730000, China (e-mail: liush20@chinatelecom.cn).}
    }

% The paper headers
\markboth{Journal of \LaTeX\ Class Files,~Vol.~14, No.~8, August~2021}%
{Shell \MakeLowercase{\textit{et al.}}: A Sample Article Using IEEEtran.cls for IEEE Journals}

\maketitle
\begin{abstract}
    Wireless foundation models promise transformative capabilities for channel state information (CSI) processing across diverse 6G network applications, yet face fundamental challenges due to the inherent dual heterogeneity of CSI across both scale and scenario dimensions. However, current pretraining approaches either constrain inputs to fixed dimensions or isolate training by scale, limiting the generalization and scalability of wireless foundation models. In this paper, we propose \textbf{HeterCSI}, a channel-adaptive pretraining framework that reconciles training efficiency with robust cross-scenario generalization via  a new understanding of gradient dynamics in heterogeneous CSI pretraining. Our key insight reveals that CSI scale heterogeneity primarily causes destructive gradient interference, while scenario diversity actually promotes constructive gradient alignment when properly managed. Specifically, we formulate heterogeneous CSI batch construction as a partitioning optimization problem that minimizes zero-padding overhead while preserving scenario diversity. To solve this, we develop a scale-aware adaptive batching strategy that aligns CSI samples of similar  scales, and design a double-masking mechanism to isolate valid signals from padding artifacts. Extensive experiments on 12 datasets demonstrate that HeterCSI establishes a generalized foundation model without scenario-specific finetuning, achieving superior average performance over full-shot baselines. Compared to the state-of-the-art zero-shot benchmark WiFo, it reduces NMSE by 7.19 dB, 4.08 dB, and 5.27 dB for CSI reconstruction, time-domain, and frequency-domain prediction, respectively.  The proposed HeterCSI framework also reduces training latency by 53\% compared to existing approaches while improving generalization performance by 1.53 dB on average.
\end{abstract}
\begin{IEEEkeywords}
Wireless foundation model, channel state information, channel prediction, CSI reconstruction
\end{IEEEkeywords}

\IEEEpeerreviewmaketitle

\section{Introduction}
\IEEEPARstart{T}{he} evolution toward sixth-generation (6G) mobile networks, characterized by massive multiple-input multiple-output (MIMO) systems, ultra-wide bandwidths, and highly dynamic environments, demands unprecedented channel prediction capabilities~\cite{gao2025enabling}. Channel State Information (CSI) prediction (i.e., estimating complete channel responses from partial observations across temporal, frequency, and spatial domains) has emerged as a critical enabler for reducing signaling overhead while maintaining communication reliability~\cite{zhang2023ai}. However, the increasing complexity and diversity of wireless scenarios present a fundamental challenge: developing foundation models that generalize robustly across heterogeneous environments while maintaining high prediction accuracy~\cite{akrout2023domain}.

Recent advances in foundation models have demonstrated remarkable generalization capabilities in natural language processing and computer vision~\cite{hong2024spectralgpt}. This success has inspired efforts to develop wireless foundation models for channel-related tasks. Current approaches fall into two categories: (1) Adapting pre-trained language or vision models (e.g., GPT) to wireless tasks via fine-tuning~\cite{liu2024llm4cp,liu2025llm4wm,guo2025lvm4csi,sun2025llm4pg}, and (2) Pretraining channel-specific models from scratch on large-scale CSI datasets~\cite{liu2025wifo,guler2025multi,alikhani2024large,yu2024channelgpt,pan2025large,yang2025wirelessgpt,catak2025bert4mimo,zhang2026wifo}. These wireless foundation models have been recently extended to multiple air-interface tasks, including channel prediction~\cite{liu2025wifo}, LoS/NLoS classification~\cite{alikhani2024large}, beamforming~\cite{liu2025llm4wm}, and wireless positioning~\cite{pan2025large}. 

% primarily because existing training paradigms fail to account for the multi-scale and multi-scenario characteristics of CSI.

\begin{figure*}[t]
    \centering
    % --- 第一行：子图 (a) 和 (b) ---
    % 调整 width 为 0.45\linewidth 左右，使两图并排占据上方
    \subfloat[Heterogeneous CSI]{%
        \includegraphics[width=0.42\linewidth]{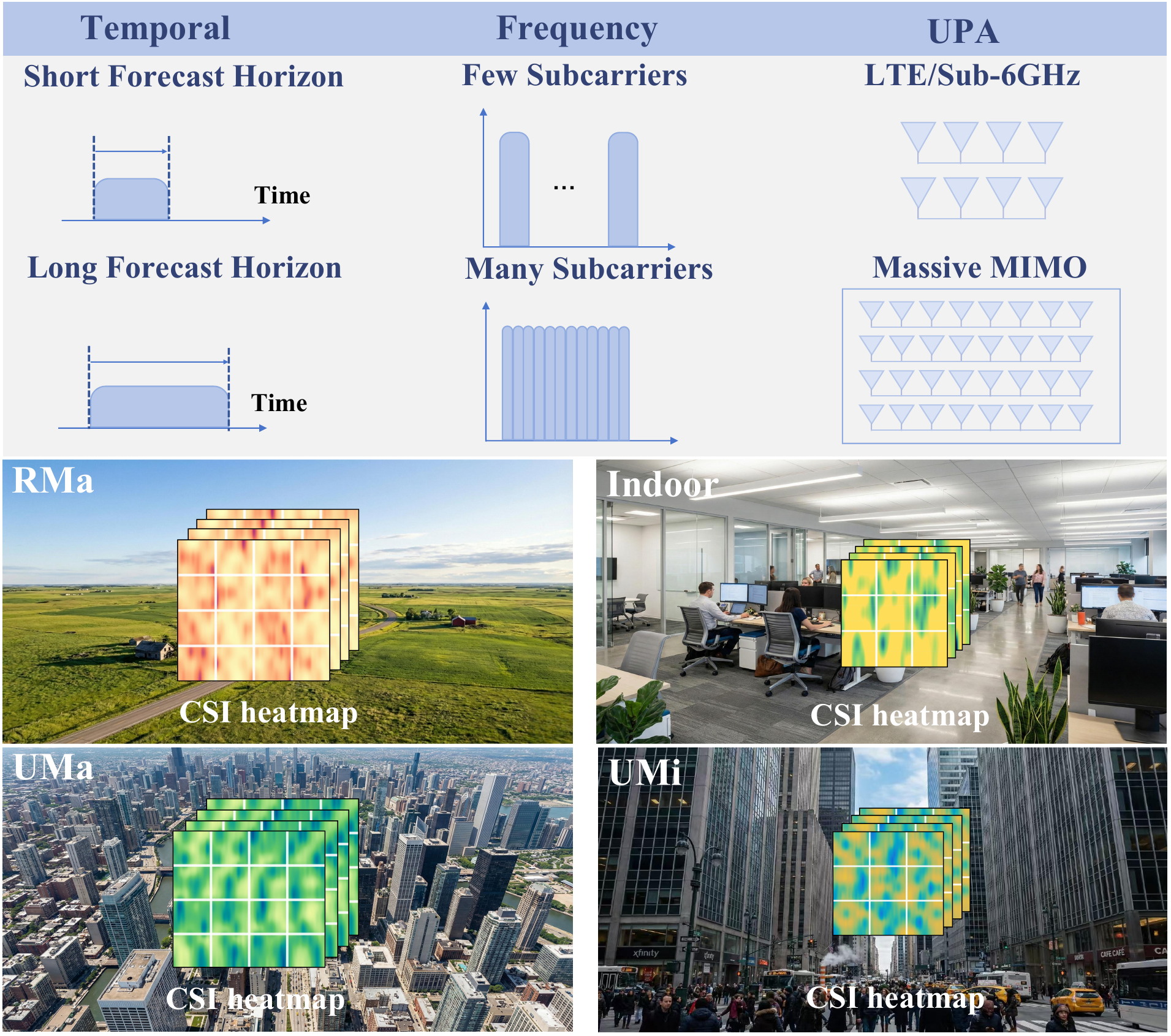}%
        \label{fig:heterogeneous_csi}%
    }%
    \hfill
    \subfloat[Gradient Interaction]{%
        \includegraphics[width=0.55\linewidth]{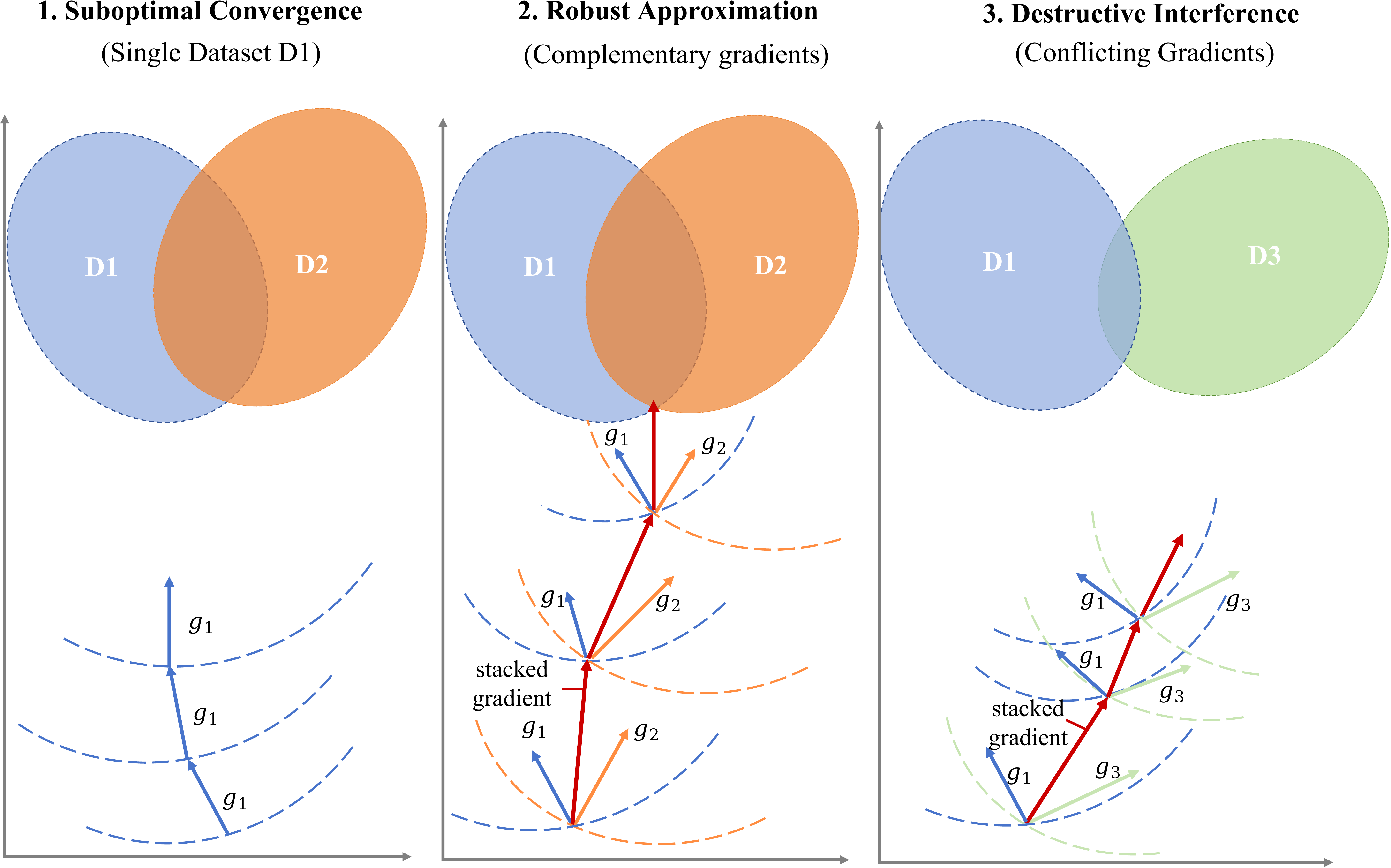}%
        \label{fig:gradient_interaction}%
    }%
    
    \vspace{1em} % 增加垂直间距，让 (c) 与上方分开
    
    % --- 第二行：子图 (c)，内部包含三张直方图 ---
    % 将原来的三张图放在同一个 \subfloat 中
    % 每张小图宽度设为 0.32\linewidth，这样三张图可以并排填满一行
    \subfloat[Gradient Cosine Similarity Distributions]{%
        \includegraphics[width=0.32\linewidth]{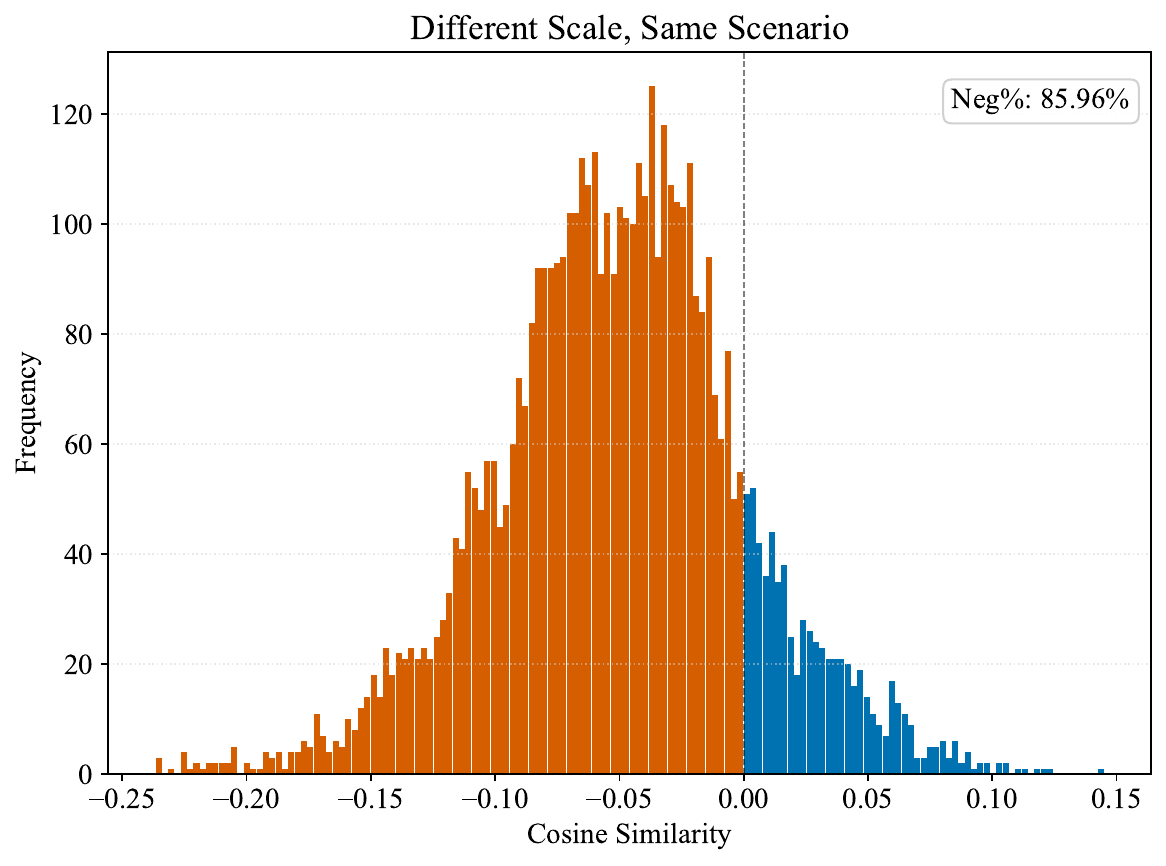}%
        \hfill
        \includegraphics[width=0.32\linewidth]{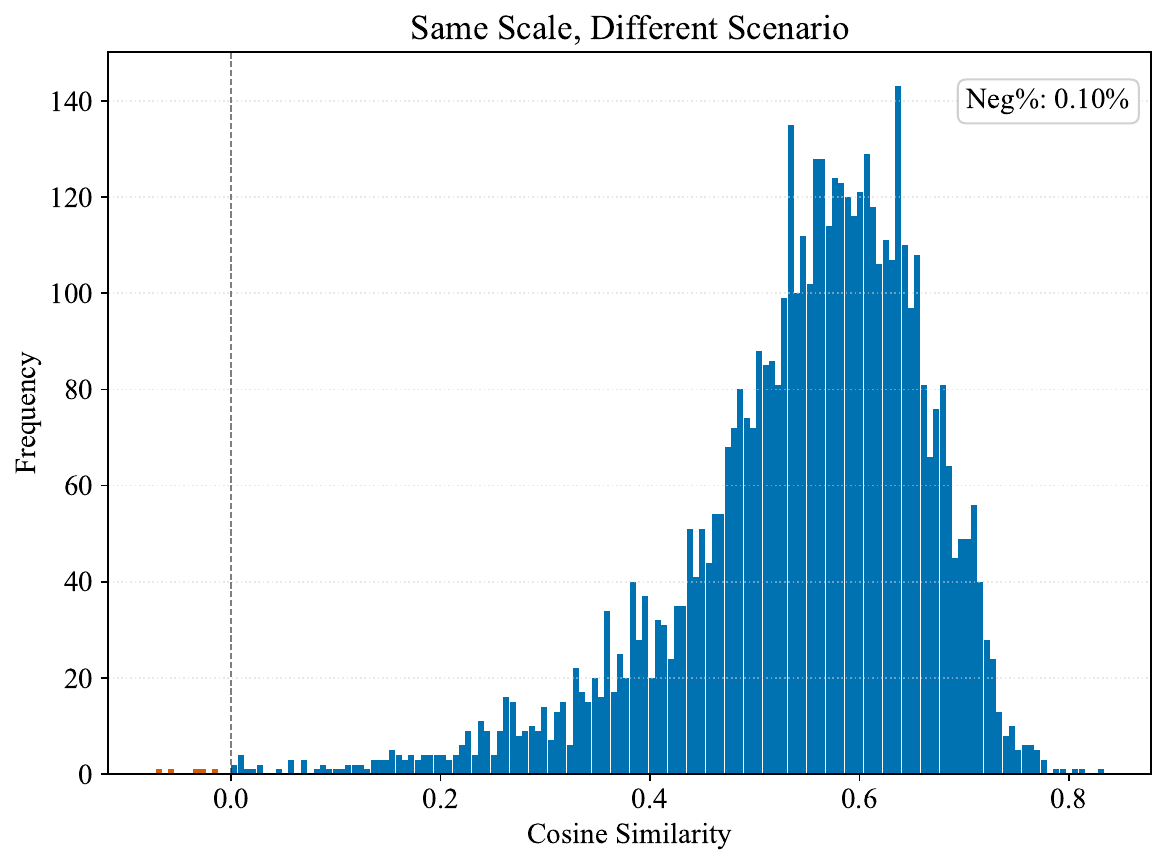}%
        \hfill
        \includegraphics[width=0.32\linewidth]{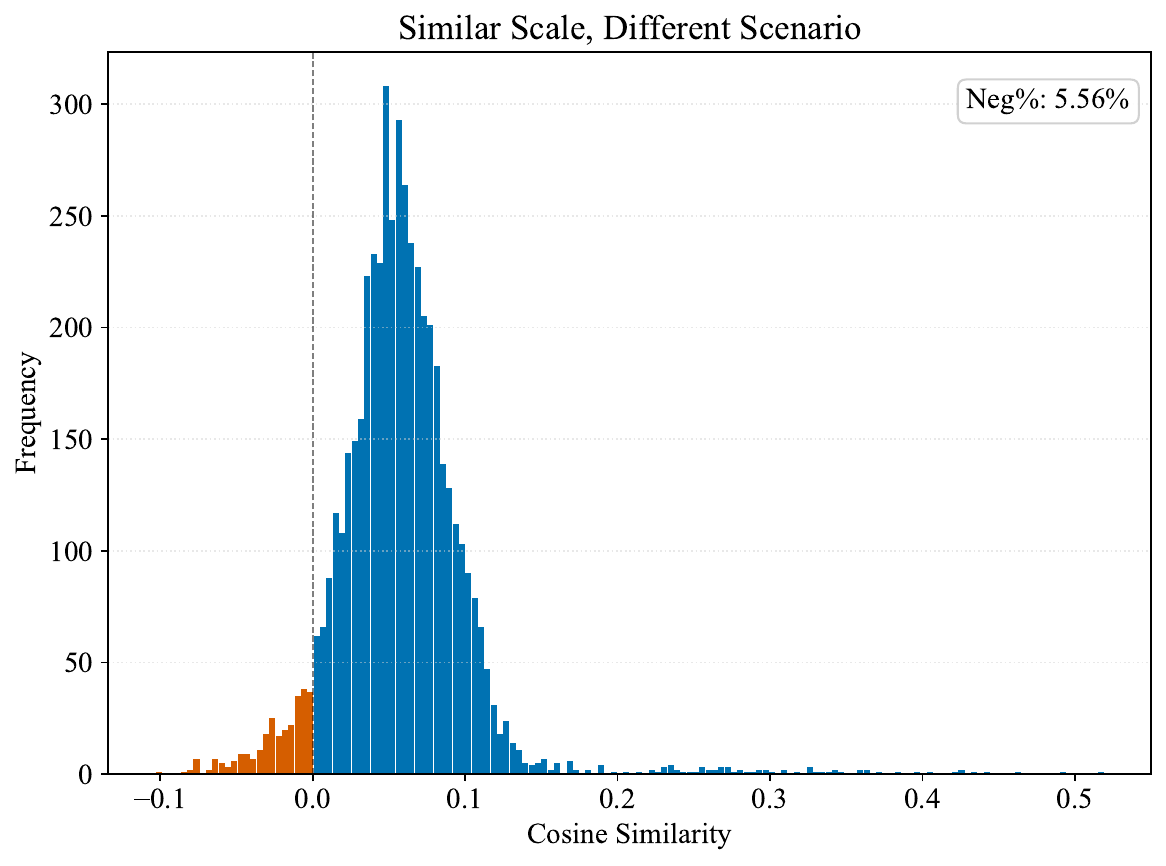}%
        \label{fig:gradient_histograms}%
    }%
    
    % --- 合并后的 Caption ---
    \caption{
    (a) CSI varies significantly across scales ($T, K, U$) and propagation scenarios. 
    (b) Jointly leveraging heterogeneous CSI aggregates complementary gradients ($g_1, g_2$) to guide optimization toward a robust global optimum, while conflicting gradients ($g_1, g_3$) lead to destructive interference. 
    (c) Similarity statistics confirm that scale discrepancy is the primary source of gradient conflict (left: scale heterogeneity), while scenario diversity maintains positive alignment when scales are compatible (middle and right).
    }
    \label{fig:motivation_combined}
\end{figure*}

\subsection{Challenge: CSI Heterogeneity and Gradient Conflict}
The challenge stems from the inherent dual heterogeneity embedded in real-world CSI data. As illustrated in Fig. 1(a), CSI exhibits variations across two orthogonal dimensions: \textit{(1) Scale Heterogeneity}, where tensor dimensions vary significantly due to different temporal horizons, subcarrier numbers, and antenna configurations (e.g.,  sub-6GHz systems with small antenna arrays versus millimeter-wave with massive MIMO); and \textit{(2) Scenario Heterogeneity}, where statistical distributions undergo substantial shifts across diverse environments (e.g., RMa, Indoor, UMa, UMi)~\cite{3gpp38901}. 

Such dual heterogeneity interacts with each other and creates conflicting optimization objectives during model training. As shown in Fig. 1 (b),  training on a single dataset often drives optimization toward suboptimal local minima. While jointly leveraging heterogeneous CSI is theoretically beneficial for aggregating complementary gradient directions toward a robust global optimum, naively mixing and shuffling multi-scale CSI data introduces a critical optimization instability: severe gradient conflict. This conflict causes orthogonal or opposing optimization directions, severely hindering the convergence of a generalized CSI model. The core research problem is therefore twofold: \textit{(1) Which factor (scale or scenario heterogeneity) is dominant in optimization instability preventing generalized CSI training? (2) How can we design a channel-adaptive mechanism to align gradients for robust cross-scale CSI pretraining?}

\subsection{Motivation and Design Rationale}
We conducted an in-depth analysis of gradient interactions under different data composition strategies. Our investigation revealed a crucial insight: \textit{CSI scale heterogeneity is the dominant source of destructive gradient interference}, while scenario diversity actually facilitates constructive gradient accumulation when scale differences are properly managed. As shown in Fig. 1(c), when samples with different CSI scales are batched together,  $85.96\%$ of gradient pairs exhibit negative cosine similarity. In contrast, when samples share identical or similar scales, the negative proportion drops dramatically to $0.01\%$ or $5.56\%$, respectively, even across diverse propagation scenarios~(e.g., RMa/UMa and UMi/Indoor).

This finding fundamentally reshapes the design rationale. Rather than treating both scale and scenario heterogeneity as detrimental, we recognize that scenario diversity is inherently beneficial for generalization, and scale heterogeneity is the primary obstacle to overcome. Based on this insight, we propose to group CSI samples into buckets according to their sizes. Within each scale-homogeneous bucket, we preserve the beneficial diversity of propagation scenarios. By stochastically sampling batches from different buckets during training, we maintain global data diversity while preventing catastrophic forgetting. This dual-level optimization strategy, i.e., enforcing intra-batch scale homogeneity while preserving inter-batch scenario diversity, can both minimize zero-padding overhead (i.e., the additional training overhead of adding zero elements to reach the same size) and align gradient directions to prevent destructive interference.

\subsection{Contribution and Paper Organization}

To establish a generalized wireless foundation model leveraging massive heterogeneous CSI, this paper proposes \textbf{HeterCSI}, a channel-adaptive pretraining framework designed to reconcile training efficiency with robust generalization. Our key insight stems from that CSI scale heterogeneity induces destructive gradient interference, whereas scenario diversity intrinsically facilitates constructive gradient alignment. Based on this, we formulate heterogeneous CSI batch construction as an optimization problem aimed at minimizing zero-padding while preserving scenario diversity. To address this, we propose a bucket-based sorting strategy to efficiently obtain an approximate solution.

The key contributions are as follows:
\begin{itemize}
    \item We reveal that scale heterogeneity is the dominant source of gradient conflict, whereas scenario diversity promotes convergence. This insight shifts the training paradigm from randomly data shuffling to active scale alignment, effectively decoupling convergence-impeding scale heterogeneity from beneficial scenario variations.
    \item We formulate heterogeneous CSI batch construction as a partitioning optimization problem aimed at minimizing zero-padding overhead subject to scenario diversity constraints. We efficiently solve this problem by introducing a scale-aware adaptive batching strategy that aligns samples of similar scales.
    \item We develop the HeterCSI framework, which integrates scale-aware adaptive batching strategy with double-masking mechanism. The former aligns samples with similar scales to alleviate gradient conflicts, while the latter leverages asymmetric attention masking to isolate valid signals from padding and to enable effective self-supervised learning.
\end{itemize}

Extensive experiments on diverse datasets generated via QuaDRiGa demonstrate the superiority of our framework from three perspectives: (1) \textit{Generalization}: Our method establishes a generalized foundation model without scenario-specific finetuning, outperforming full-shot baselines on average across 12 unseen zero-shot scenarios. Compared to the state-of-the-art zero-shot benchmark WiFo, it significantly reduces the average NMSE by 7.19~dB, 4.08~dB, and 5.27~dB for CSI reconstruction, time-domain prediction, and frequency-domain prediction tasks, respectively; (2) \textit{Scalability}: Evaluations with increasing training scales (expanding from 8 to 40 datasets) confirm that generalization performance improves consistently with data volume, while improving generalization by an average of 1.53 dB; and (3) \textit{Efficiency}: The proposed scale-aware adaptive batching strategy reduces training time by approximately 53\% compared to global shuffling. The source code is available at the GitHub repository \url{https://github.com/zcy8998/HeterCSI/}.

The rest of this paper is organized as follows. Section 2 reviews related work and analyzes the limitations of existing methods. Section 3 formulates and analyzes the padding minimization problem subject to the scenario diversity. Section 4 details the proposed HeterCSI framework. Section 5 validates the zero-shot performance and scalability of our approach, and Section 6 concludes the paper.

\begin{table*}[!t]
    \renewcommand{\arraystretch}{1.25} % 增加行高
    \setlength{\tabcolsep}{0pt} % 列间距设为0，由 extracolsep 自动填充
    \caption{Comparative Analysis of Wireless Foundation Models in Terms of Training Paradigm, CSI Tasks, and Generalization Capabilities}
    \label{tab:relatework}
    \centering
    \footnotesize
    
    % 定义灰色背景
    \definecolor{Gray}{gray}{0.9}
  
    \begin{tabular*}{\textwidth}{@{\extracolsep{\fill}}l l ccc cc}
      \toprule
      \multirow{2}{*}{\textbf{Method}} & \multicolumn{1}{c}{\textbf{Training Paradigm}} & \multicolumn{3}{c}{\textbf{CSI Task Support}} & \multicolumn{2}{c}{\textbf{Generalization}} \\
      \cmidrule(lr){2-2} \cmidrule(lr){3-5} \cmidrule(lr){6-7}
       & \multicolumn{1}{c}{(Methodology \& Architecture)} & Recon. & Time Pred. & Freq. Pred. & Freq. & Scenario \\
      \midrule
      
      LLM4CP~\cite{liu2024llm4cp}            & Fine-Tuning (Fixed-scale, Single Scenario)     & \xmark & \cmark & \xmark & \cmark & \cmark \\
      LLM4WM~\cite{liu2025llm4wm}            & Fine-Tuning (Fixed-scale, Single Scenario)     & \xmark & \cmark & \cmark & \cmark & \cmark \\
      LVM4CSI~\cite{guo2025lvm4csi}          & Fine-Tuning (Fixed-scale, Single Scenario)               & \cmark & \xmark & \xmark & \xmark & \xmark \\
      LWM~\cite{alikhani2024large}           & Fixed-scale and Single Scenario Specific Pretraining & \cmark & \xmark & \xmark & \xmark & \cmark \\
      LWLM~\cite{pan2025large}               & Fixed-scale and Single Scenario Specific Pretraining & \cmark & \xmark & \xmark & \cmark & \xmark \\
      WiMAE~\cite{guler2025multi}            & Fixed-scale and Single Scenario Specific Pretraining & \cmark & \xmark & \xmark & \xmark & \cmark \\
      ChannelGPT~\cite{yu2024channelgpt}     & Fixed-scale and Single Scenario Specific Pretraining & \cmark & \cmark & \xmark & \xmark & \cmark \\
      WirelessGPT~\cite{yang2025wirelessgpt} & Fixed-scale and Single Scenario Specific Pretraining & \cmark & \cmark & \xmark & \cmark & \cmark \\
      BERT4MIMO~\cite{catak2025bert4mimo}    & Fixed-scale and Single Scenario Specific Pretraining & \cmark & \xmark & \cmark & \xmark & \cmark \\
      WiFo~\cite{liu2025wifo}                & Scale- and Scenario-isolated
      Pretraining   & \cmark & \cmark & \cmark & \cmark & \cmark \\
      \midrule
      
      % 突出显示 Proposed 方法
      \rowcolor{Gray}
      \textbf{Proposed} & \textbf{Cross Scale and Scenario Shuffling Pre-Training} & \cmark & \cmark & \cmark & \cmark & \cmark \\
      \bottomrule
    \end{tabular*}
\end{table*}

\section{Related Work}
This section reviews the evolution of CSI prediction from small models to emerging wireless foundation models, highlighting their limitations. We also review the impact of batch composition and shuffling strategies on generalization, and the challenges posed by heterogeneous CSI data. Addressing these gaps, our work is the first to enable cross-scale shuffled pretraining, paving the way for a generalized wireless foundation model.

\subsection{Conventional and Small-Model Methods}
Traditional CSI prediction methods rely on analytical channel models that capture physical phenomena such as multipath propagation, fading, and noise~\cite{li2025bridging}.  Classical estimators, including least-squares (LS)~\cite{hussein2023least}, minimum mean square error (MMSE)~\cite{bacci2024mmse}, Prony-based techniques~\cite{yin2020addressing}, and Kalman filtering~\cite{kim2020massive}, have been widely adopted for channel estimation and prediction. While effective in low-mobility, sparse-channel environments, these methods suffer from strong dependence on idealized assumptions and fail to capture the nonlinear, high-dimensional dynamics of modern massive MIMO systems.

To overcome these limitations, data-driven AI-based methods have emerged. Neural network architectures such as CNNs~\cite{xia2019deep}, LSTMs~\cite{jiang2020deep}, GRUs~\cite{helmy2023lstm}, and Transformers~\cite{jiang2022accurate} learn spatio-temporal patterns directly from CSI data, enabling end-to-end prediction without explicit channel modeling. However, these models typically require large amounts of task-specific training data and exhibit poor generalization to unseen scenarios (e.g., different frequencies, array sizes, or mobility patterns). As a result, they remain impractical for dynamic 6G deployments where channel conditions vary widely, and zero-shot prediction capability is of paramount importance.

\subsection{Wireless Channel Foundation Models}
Recent efforts have explored foundation models for wireless communications, aiming to learn universal representations of channel behavior. 

\subsubsection{Adapting Pre-trained Models for Wireless Tasks}
This category adapts pretrained language or vision models (e.g., GPT, ViT) to wireless tasks via fine-tuning~\cite{liu2024llm4cp,liu2025llm4wm,guo2025lvm4csi,sun2025llm4pg}. For example, LLM4CP~\cite{liu2024llm4cp} freezes a GPT-2 backbone and transfers knowledge to CSI prediction through adapter modules. LVM4CSI~\cite{guo2025lvm4csi} treats CSI as an image and leverages pretrained vision transformers (ViTs). While promising, these methods face a modality gap: CSI is complex-valued and structured across space–time–frequency, unlike natural text or RGB images, limiting transfer effectiveness.

\subsubsection{Pretraining from Scratch}
This category pretrains channel-specific foundation models from scratch on large-scale CSI datasets~\cite{liu2025wifo,guler2025multi,alikhani2024large,yu2024channelgpt,pan2025large,yang2025wirelessgpt,catak2025bert4mimo,zhang2026wifo}. WiFo~\cite{liu2025wifo} is a notable example: it supports multi-scale CSI inputs and unifies time- and frequency-domain prediction under a single architecture. However, WiFo trains on each scale in isolation and does not perform cross-dataset shuffling during pretraining. Other models like WirelessGPT~\cite{yang2025wirelessgpt}, ChannelGPT~\cite{yu2024channelgpt}, and BERT4MIMO~\cite{catak2025bert4mimo} further assume fixed-size inputs or homogeneous training data. Table~\ref{tab:relatework} provides a comparative analysis of channel foundation models from the perspectives of training paradigm, CSI tasks, and generalization capabilities.

\subsection{Batching and Shuffling for Heterogeneous CSI}

\subsubsection{Batch Optimization and Active Sampling}
The composition and shuffling strategy of training mini-batches critically shape model convergence and generalization. Early studies attribute the generalization advantage of small-batch SGD to gradient noise that biases optimization toward flat minima~\cite{keskar2016large}. More recent work shows that increasing dataset volumes reshapes the loss landscape, enabling large-batch training to converge to sharp yet well-generalizing minima~\cite{fan2025sharp}. The shuffling mechanism ensures randomness in data access, preventing the model from memorizing specific order of the training data and thereby promoting better generalization~\cite{hardt2016train}. However, in multi-task settings, naive batch construction often induces gradient conflict due to misaligned task-specific gradients, leading to destructive interference and degraded optimization~\cite{shi2024conflict}. 

To enhance training efficiency and representation quality, current training pipelines increasingly transition from uniform sampling to adaptive, data-centric batch construction strategies. Recent advances in Large Language Model (LLM) adopt quality-aware data selection, which dynamically filters instructions based on complexity and educational value to accelerate convergence~\cite{li2024quantity}. In computer vision, contrastive frameworks employ adaptive curriculum learning that progressively introduces samples with higher semantic uncertainty to refine feature discrimination without destabilizing the training process~\cite{wang2023comprehensive}. Furthermore, to mitigate gradient conflict in multi-task learning, current methods have transitioned from expensive gradient projections to efficient multi-objective optimization schemes that balance task-specific gradients via conflict-free aggregation~\cite{shi2024conflict}. Despite these advances, these strategies are primarily tailored to address semantic complexity or optimization imbalances and lack the necessary mechanisms to handle the dual heterogeneity inherent in CSI.

\subsubsection{Heterogeneity in Channel Foundation Models}
Training generalized channel foundation models requires organizing heterogeneous CSI into batches with unified scale. Standard training pipelines for LLMs commonly rely on sequence packing~\cite{brown2020language}, padding~\cite{vaswani2017attention}, and truncation~\cite{devlin2018bert} to accommodate variable-length inputs. However, directly transferring these strategies to CSI is non-trivial. While truncation and packing violate physical integrity by discarding essential signal components~\cite{bajwa2010compressed} or disrupting delay–Doppler correlations, naive zero-padding is equally ill-suited due to the extreme scale heterogeneity of wireless channels. For instance, in 5G NR, the maximum number of Channel State Information–Reference Signal (CSI-RS) antenna ports is 32, which determines the amount of CSI generated; whereas in 6G architectures, up to 256 antenna ports are expected to be deployed~\cite{samsung2025_6g_mimo}.

As illustrated in the \textbf{Training Paradigm} of Table I, most existing methods~\cite{liu2024llm4cp,liu2025llm4wm,guler2025multi,alikhani2024large,yu2024channelgpt,pan2025large,yang2025wirelessgpt,catak2025bert4mimo} assume fixed-size CSI inputs or train on scale-homogeneous data, limiting their adaptability to real-world heterogeneity. A notable recent effort, WiFo~\cite{liu2025wifo}, recognizes the need for multi-scale support and trains a single model on CSI with varying $T$, $K$, and $N$. In this paper, we further enhance representation learning by enabling multi-scale interactions within training batches. While WiFo is capable of handling multi-scale data, its training paradigm remains scale- and scenario-isolated training, where each mini-batch is restricted to a single dataset. In contrast, we reveal that scale heterogeneity is the dominant source of gradient conflict, whereas scenario diversity promotes convergence. This insight shifts the training paradigm from randomly data shuffling to active scale alignment and allows for joint exposure to diverse CSI scales, unlocking the full potential of generalized foundation models.

\section{Cross-Scale Pretraining Challenges and Problem Formulation}
In this section, we first analyze the mathematical characteristics of heterogeneous CSI data. By explicitly linking scale diversity to padding-induced inefficiency and gradient interference, we formulate a batch construction problem to reconcile training efficiency with robust generalization.

\subsection{Scale Heterogeneity of CSI}
\label{subsec:physical_heterogeneity}

The evolution of wireless standards, including 4G LTE, 5G (Sub-6 GHz, mmWave), and visionary 6G (Sub-THz, THz) systems—introduces profound heterogeneity in frequency bands and antenna configurations~\cite{3gpp_ts38_101_1,itu_r_imt2030}. Furthermore, intra-standard flexibility allows for diverse carrier bandwidths, subcarrier spacings, and array sizes. Collectively, these variations fundamentally govern the multi-scale nature of CSI data along the frequency, temporal, and spatial dimensions:

\begin{itemize}
    \item \textbf{Frequency Domain ($K$):} Given a fixed subcarrier spacing, the number of subcarriers $K$ scales linearly with bandwidth~\cite{dahlman20205g}. Deployment scenarios range from narrowband sub-6GHz systems to ultra-wideband mmWave or sub-THz systems, creating significant variance in spectral dimensionality.
    \item \textbf{Spatial Domain ($A$):} The antenna array size $A = A_h \times A_v$ varies from small-scale arrays in user equipment to massive MIMO configurations (e.g., 256 ports) in base stations.
    \item \textbf{Temporal Domain ($T$):} The observation window $T$ required to capture channel dependencies fluctuates with user mobility and coherence time, functioning as a scenario-dependent variable.
\end{itemize}

Consequently, a CSI sample from a specific scenario is formalized as a complex-valued tensor $\mathbf{H} \in \mathbb{C}^{T \times K \times A}$, where the tuple $(T, K, A)$ follows a highly non-uniform distribution across the global dataset.

\subsection{The Efficiency-Generalization Dilemma}
\label{subsec:dilemma}
We consider a global data collection $\mathcal{X} = \bigcup_{n=1}^{N} \mathcal{D}_n$, aggregated from $N$ heterogeneous datasets. To enable a unified foundation model to process these heterogeneous CSI, we employ a 3D patching projection. This operation transforms the raw CSI tensor of sample $x_i$ into a flattened sequence of tokens with length $L(x_i)$, calculated as:
\begin{equation}
L(x_i)=\left\lceil\frac{T_{i}}{t}\right\rceil \times \left\lceil\frac{K_{i}}{k}\right\rceil \times \left\lceil\frac{A_{i}}{a}\right\rceil,
\end{equation}
where $(t, k, a)$ denotes the fixed patch size. Due to the scale heterogeneity, $L(x_i)$ exhibits significant variance. In standard parallel training, samples in a mini-batch $\mathcal{B}_j$ are zero-padded to the batch-wise maximum length $h_j = \max_{x_i \in \mathcal{B}_j} L(x_i)$.

To quantify the computational inefficiency, we define the total padding overhead $J_{\mathrm{pad}}$ as the cumulative sum of invalid tokens across all batches:
\begin{equation}
    J_{\mathrm{pad}}(\mathbb{B}) = \sum_{j=1}^{N_{\mathrm{batch}}} \sum_{x_i \in \mathcal{B}_j} \left( h_j - L(x_i) \right).
\end{equation}

This padding poses a dilemma between efficiency and generalization. Under naive global shuffling, heterogeneous CSI samples are randomly mixed; while this mitigates distribution shift to favor generalization, it inevitably maximizes the padding overhead $J_{\mathrm{pad}}$ and exacerbates scale-induced gradient conflicts. Conversely, scale- and scenario-isolated training while effectively minimizing $J_{\mathrm{pad}}$, restricts the model to local scale-specific statistics, impairing its ability to capture cross-scale correlations and consequently degrading overall generalization performance.

\subsection{Problem Formulation and Analysis}
\label{sec:problem_formulation}

\subsubsection{Problem Formulation}
Our objective is to construct training batches that minimize padding overhead (to ensure efficiency and scale consistency) while maintaining sufficient data source diversity (to ensure generalization). We formulate the batch construction problem as partitioning the global sample set $\mathcal{X}$ into $N_{batch}$ non-overlapping mini-batches $\mathbb{B} = \{\mathcal{B}_1, \dots, \mathcal{B}_{N_{batch}}\}$ of fixed size $N_{bs}$. The problem is formulated as:

\begin{equation}
    \begin{aligned}
    \textbf{P:} \quad \min_{\mathbb{B}} \quad & J_{pad}(\mathbb{B}) \\
    \textrm{s.t.} \quad & \text{C1: } \bigcup_{j=1}^{N_{batch}} \mathcal{B}_j = \mathcal{X}, \quad \mathcal{B}_j \cap \mathcal{B}_k = \emptyset, \quad \forall j \neq k, \\
    & \text{C2: } |\mathcal{B}_j| = N_{bs}, \quad \forall j, \\
    & \text{C3: } \mathcal{H}(\{d(x_i) \mid x_i \in \mathcal{B}_j\}) \geq \epsilon, \quad \forall j.
    \end{aligned}
    \label{eq:general_problem}
\end{equation}
where $d(x_i)$ maps a sample $x_i$ to its source dataset index, and $\mathcal{H}(\cdot)$ denotes a metric that quantifies the diversity of data sources within a batch. Constraint C3 ensures that the distribution of scenarios in each mini-batch remains sufficiently diverse (above the threshold $\epsilon$) to prevent the model from overfitting to specific statistical domains.

\subsubsection{Problem Analysis}

Problem \textbf{P} is formally a constrained set partitioning problem~\cite{balas1976set}. Given the large volume of data in dataset $\mathcal{X}$, finding an exact global optimum is NP-hard. The challenge arises from the combinatorial explosion of assigning samples to batches. Specifically, batch assignments are globally interdependent due to the strict partitioning enforced by C1. Furthermore, there is an inherent conflict between the objectives: minimizing padding typically requires grouping samples of similar scale, whereas constraint C3 requires grouping samples from diverse sources (which may have varied scale). To address this computational intractability, we propose a hierarchical approach that transforms $\textbf{P}$ into a coarse-grained bucket partitioning problem, followed by a stochastic sampling strategy.

\begin{enumerate}
    \item We first relax the strict mini-batch constraints to partition the global dataset into $B$ buckets. The objective transforms into minimizing the potential padding overhead within each bucket while guaranteeing diversity at the bucket level. This sub-problem is formulated as:
    
    \begin{equation}
    \begin{aligned}
    \textbf{P}': \quad & \min_{\mathbb{U}} J_{pad}(\mathbb{U}) \\
    \textrm{s.t.} \quad & \bigcup_{k=1}^{B} b_k = \mathcal{X}, \quad b_i \cap b_j = \emptyset, \\
    & \mathcal{H}(\{d(x_{i}) | x_{i} \in b_{k}\}) \ge \epsilon, \quad \forall k.
    \end{aligned}
    \end{equation}
    
    By solving $\textbf{P}'$, we obtain a set of scale-similar buckets $\mathbb{U}^*$, where samples within the same bucket share similar scale yet originate from diverse scenarios.

    \item Given the optimal partition $\mathbb{U}^{*}$ from $\textbf{P}'$, the problem reduces to constructing the final mini-batches $\mathcal{B}$ from these coarse-grained buckets. The objective is to devise an assignment strategy that distributes samples from $\mathbb{U}^{*}$ into batches, ensuring that the strict scenario diversity constraint C3 of the original problem $\textbf{P}$ is recovered while preserving the scale consistency achieved in the relaxation phase.
\end{enumerate}

\begin{figure*}[!t]
\centering
\includegraphics[width=\linewidth,keepaspectratio]{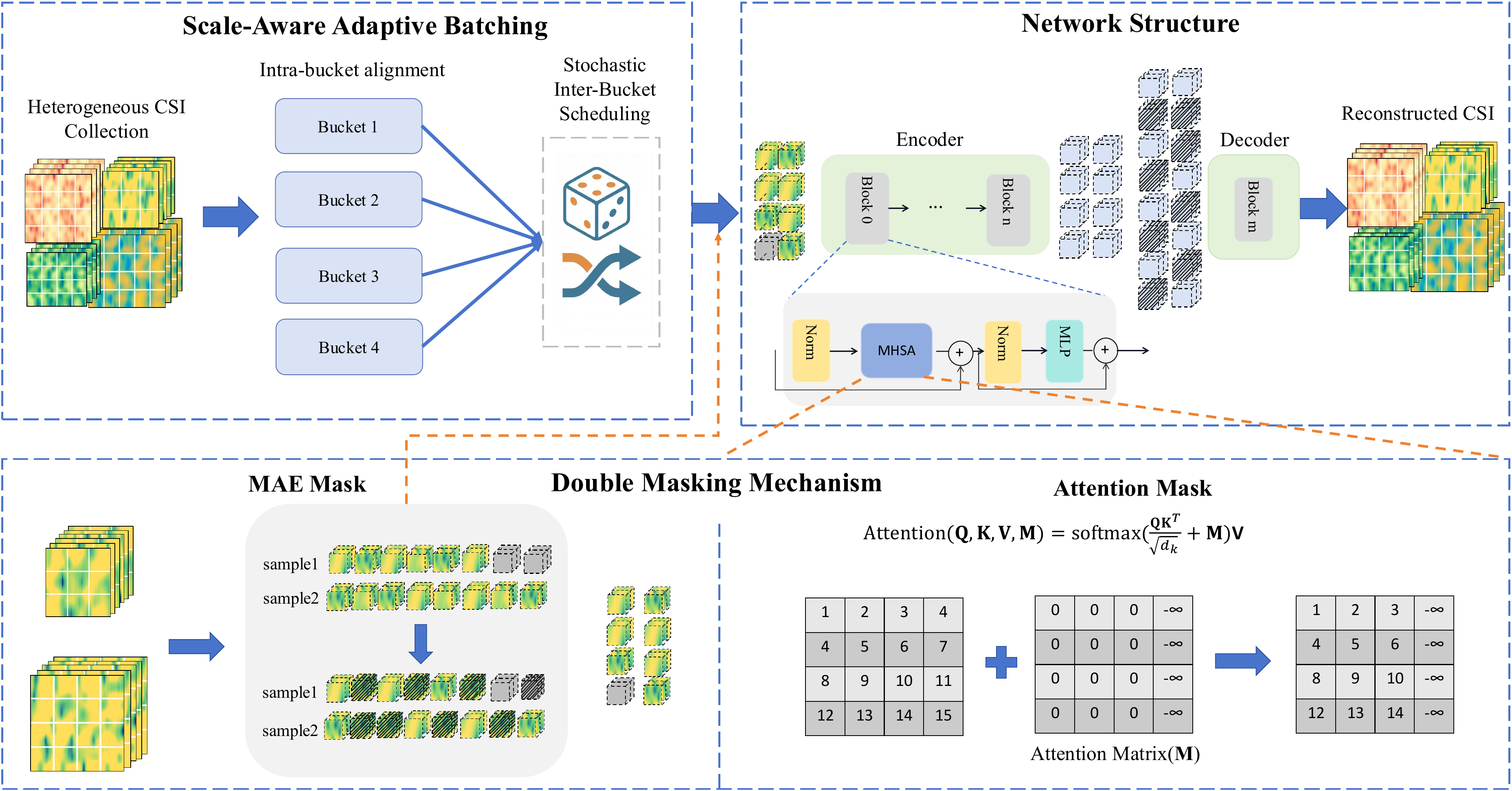}
\caption{Overview of the HeterCSI framework. The framework introduces \textit{scale-aware adaptive batching strategy}, which synergizes intra-bucket gradient alignment with inter-bucket stochastic scheduling. Furthermore, a \textit{double masking mechanism} is designed to rigorously isolate valid signals from padding artifacts via attention masking, while simultaneously driving self-supervised representation learning through MAE random masking. The processed tokens are finally reconstructed via a standard Transformer-based encoder-decoder architecture.}
\label{fig:framework}
\end{figure*}

\section{Channel-Adaptive Heterogeneous CSI Pretraining Framework}

This section presents a channel-adaptive self-supervised learning framework. We begin with an overview of the proposed framework, followed by a detailed introduction to the scale-aware adaptive batching strategy and the double-masking mechanism. Finally, we provide a description of the network architecture and the associated CSI tasks.
\subsection{Overview of Proposed Framework}
Our approach is motivated by the need for a training paradigm that achieves both efficiency and strong generalization when learning from multi-scale, multi-scenario CSI data. We propose a channel-adaptive pre-training framework, as illustrated in Fig.~\ref{fig:framework}. Specifically, we address the following key challenges:

\textit{1) Generalization Loss via Sequential Bias:} 
Naive sequential training on heterogeneous CSI induces severe distribution shifts. This triggers catastrophic forgetting of prior scenarios and local overfitting, fundamentally compromising the model's ability to generalize across diverse environments.

\textit{2) Cross-Scale Gradient Conflict:} 
Batching CSI samples with vastly different scales creates conflicting gradient directions. This interference destabilizes the optimization landscape, hindering convergence and the learning of robust, scale-invariant representations.

\textit{3) Padding Overhead and Feature Dilution:} 
While data shuffling aids generalization, it necessitates heavy zero-padding for multi-scale CSI data. This introduces prohibitive computational waste and dilutes valid features with non-informative tokens, degrading extraction efficiency.

To address these challenges, we propose the following designs:

\begin{itemize}
    \item \textbf{Scale-Aware Adaptive Batching Strategy:}
    We develop a scale-aware adaptive batching strategy that combines intra-bucket alignment to minimize padding overhead while preserving intrinsic intra-bucket diversity, with inter-bucket stochastic scheduling to maintain global data diversity and ensure robust generalization across heterogeneous scenarios.

    \item \textbf{Double Masking Mechanism:}
    We adopt a double-masking mechanism that integrates MAE~\cite{he2022masked} with attention masking. The former enables self-supervised reconstruction of CSI, while the latter prevents padded tokens from interfering with attention, allowing the model to focus on valid CSI structures.
\end{itemize}

\subsection{Scale-Aware Adaptive Batching Strategy}
To address problem \textbf{P}, we propose a scale-aware adaptive batching strategy. Our design rationale stems from the decomposition of the dual heterogeneity: while scale discrepancy constitutes the primary barrier to computational efficiency and gradient stability, scenario diversity remains the cornerstone of generalization. Consequently, we treat Problem $\textbf{P}$ as a two-stage decision process: (1) intra-bucket alignment to strictly minimize $J_{pad}$ by clustering scale-compatible samples, effectively solving the relaxation $\textbf{P}'$; and (2) inter-bucket stochastic scheduling to asymptotically satisfy the diversity constraint $C3$ via stochastically sampling from these pre-sorted buckets.

\subsubsection{Intra-Bucket Alignment}
Problem \textbf{P'} aims to determine a bucket partition strategy that minimizes the total padding overhead subject to the scenario diversity. To solve this problem, we propose a sort-and-partition heuristic that efficiently approximates the optimal solution, which operates on the global dataset $\mathcal{X}$ and involves the following two steps:
\begin{enumerate}
    \item \textbf{Global Sorting:} Instead of processing datasets in isolation, we aggregate all samples into a unified pool and sort them based on their flattened sequence length $L(x_i)$. This creates a globally ordered sequence $\mathcal{S}_{ordered}$ where adjacent samples possess high dimensional similarity. 
    
    \item \textbf{Bucket Partitioning:} We discretize the continuous sorted sequence into a fixed number of buckets, denoted by the hyperparameter $B$. Specifically, $\mathcal{S}_{ordered}$ is sliced into $B$ equal-sized buckets $\mathbb{U} = \{b_1, \dots, b_B\}$, where each bucket contains approximately $|\mathcal{X}|/B$ samples.
\end{enumerate}

This strategy ensures that samples within each bucket $b_k$ have minimal length variance (minimizing $J_{pad}$), while the global shuffling during sorting ensures that the source entropy $\mathcal{H}(\{d(x_i) \mid x_i \in b_k\})$ remains high (satisfying the constraint). In the following, we formally prove that, under a fixed number of buckets, the proposed intra-bucket alignment strategy achieves the minimum padding overhead.

\begin{theorem}
    Given a fixed number of partitions, the strategy of grouping datasets by sorted lengths yields a globally optimal solution that minimizes the total padding overhead.
\end{theorem}

\begin{proof}
    We employ the principle of the rearrangement inequality. Note that minimizing total padding $\sum (h_j - L(x_i))$ is equivalent to minimizing the sum of batch maximums $\sum h_j$. Assume a non-optimal partition exists containing a crossed pair: a long sequence $L_{long}$ in a batch with a small maximum $h_{small}$, and a short sequence $L_{short}$ in a batch with a large maximum $h_{large}$ (where $L_{short} < L_{long}$ and $h_{small} < h_{large}$).
    
    Swapping these sequences moves $L_{long}$ to the batch with $h_{large}$. Since $L_{long} \le h_{large}$ (in a sorted context), the maximum of the larger batch does not increase. Crucially, removing $L_{long}$ from the smaller batch allows $h_{small}$ to decrease (or stay the same). Thus, the objective function $\sum h_j$ is monotonically reduced or preserved by uncrossing partitions. This concludes the proof.
\end{proof}
\begin{algorithm}[t]
    \caption{HeterCSI: Scale-Aware Adaptive Batching}
    \label{alg:hetercsi}
    \begin{algorithmic}[1]
    \REQUIRE Global Datasets $\mathcal{X}=\{\mathcal{D}_1, \dots, \mathcal{D}_N\}$, Batch Size $N_{bs}$, Buckets $B$, Training Steps $T_{total}$
    \ENSURE Pre-trained Model Parameters $\theta$
    
    \STATE \textbf{// Phase 1: Intra-Bucket Alignment}
    \STATE Aggregate all samples: $\mathcal{X}_{pool} \leftarrow \bigcup_{n=1}^N \mathcal{D}_n$.
    \STATE Calculate token sequence length $L(x_i)$ for each $x_i \in \mathcal{X}_{pool}$.
    \STATE Sort samples to obtain ordered sequence $\mathcal{S}_{ordered}$.
    \STATE Determine bucket capacity $C \leftarrow \lceil |\mathcal{X}_{pool}|/B \rceil$.
    \FOR{$k = 1$ to $B$}
        \STATE Partition stream into scale-similar bucket:
        \STATE $b_k \leftarrow \mathcal{S}_{ordered}[(k-1)C : kC]$.
    \ENDFOR
    
    \STATE \textbf{// Phase 2: Inter-Bucket Stochastic Scheduling}
    \STATE Initialize step $t \leftarrow 0$.
    \WHILE{$t < T_{total}$}
        \STATE Initialize batch pool $\mathcal{B}_{pool} \leftarrow \emptyset$
        \FOR{$k = 1$ to $B$}
            \STATE Shuffle samples within bucket $b_k$.
            \STATE Split $b_k$ into mini-batches and add to $\mathcal{B}_{pool}$.
        \ENDFOR
        \STATE Shuffle $\mathcal{B}_{pool}$ globally to mix scales and scenarios.
        
        \FOR{each mini-batch $\mathcal{M} \in \mathcal{B}_{pool}$}
            \STATE Compute batch-wise length $h \leftarrow \max_{x \in \mathcal{M}} L(x)$.
            \STATE Pad samples in $\mathcal{M}$ to length $h$ to obtain input $\mathbf{X}$.
            \STATE Update $\theta \leftarrow \text{Optimizer}(\theta, \nabla \mathcal{L}(\mathbf{X}))$.
            \STATE $t \leftarrow t + 1$.
            \IF{$t \geq T_{total}$}
                \STATE \textbf{break}.
            \ENDIF
        \ENDFOR
    \ENDWHILE
    \end{algorithmic}
\end{algorithm}

\subsubsection{Inter-Bucket Stochastic Scheduling}
Given the optimal partition $\mathbb{U}^{*}$ from $\textbf{P}'$, we construct the final mini-batches. Instead of performing computationally expensive combinatorial optimization, we solve the assignment via stochastic sampling. Since the buckets in $\mathbb{U}^*$ are pre-conditioned to be diverse, random sampling within these buckets ensures that the resulting mini-batches $\mathcal{B}_j$ satisfy the original diversity constraint C3 of Problem $\textbf{P}$ in expectation, while preserving the scale consistency achieved in the first phase.

We formally justify that the constructed mini-batches satisfy the scenario diversity constraint. Let $\mathbf{p}_k^{\text{bucket}} \in \mathbb{R}^N$ denote the probability vector representing the proportion of the $N$ datasets within bucket $b_k$. When a mini-batch $\mathcal{B}_j$ is formed via uniform random sampling from $b_k$, the probability of selecting a sample from any specific dataset $n$ equals its proportion in the bucket. Thus, the expected distribution of the mini-batch $\mathbf{p}_j^{\text{batch}}$ serves as an unbiased estimator of the bucket's distribution:
\begin{equation}
    \mathbb{E}[\mathbf{p}_j^{\text{batch}}] = \mathbf{p}_k^{\text{bucket}}.
    \label{eq:expectation}
\end{equation}
Consequently, the expected diversity of data sources within a batch converges to the corresponding bucket:
\begin{equation}
    \mathbb{E}[\mathcal{H}(\mathcal{B}_j)] \approx \mathcal{H}(b_k) \ge \epsilon.
    \label{eq:entropy}
\end{equation}
This derivation proves that our stochastic scheduling strategy effectively transfers the diversity guarantee from the coarse-grained bucket level (secured by the bucket shuffling in Phase 1) to the fine-grained batch level (required by problem $\textbf{P}$). This mechanism ensures robust generalization by exposing the model to diverse scenarios over time, all while maintaining the high computational efficiency of scale-similar batching.

\subsection{Double Masking Mechanism Adaptation}
The proposed framework incorporates a double masking mechanism that serves two objectives: facilitating self-supervised representation learning and ensuring computational validity under variable-length batched processing.

\subsubsection{MAE Masking for Representation Learning} 
Following the MAE paradigm~\cite{he2022masked}, we construct our framework based on an asymmetric encoder-decoder architecture with a self-supervised reconstruction objective. With a portion of the input CSI masked in advance, the encoder processes only the visible tokens to learn semantic latent representations. The decoder then reconstructs the missing signals based on these representations, facilitating self-supervised learning of the underlying CSI structure. To faithfully capture the complex three-dimensional variability of CSI, we adopt the masking strategies introduced in the WiFo framework~\cite{liu2025wifo}, including \textit{random masking}, \textit{time-domain masking}, and \textit{frequency-domain masking}. Each strategy targets a distinct structural aspect: random masking promotes global dependency modeling, time-domain masking enforces temporal causality awareness, and frequency-domain masking strengthens spectral correlation learning.
\subsubsection{Attention Mask Formulation}
We formally define the attention mechanism incorporating the padding. Let $\mathbf{Q}, \mathbf{K}, \mathbf{V} \in \mathbb{R}^{N_{BS} \times L \times d}$ denote the query, key, and value matrices, respectively. The masked scaled dot-product attention is computed as:

\begin{equation}
    \text{Attention}(\mathbf{Q}, \mathbf{K}, \mathbf{V}, \mathbf{M}) = \text{softmax}\left( \frac{\mathbf{Q}\mathbf{K}^\top}{\sqrt{d_k}} + \mathbf{M} \right) \mathbf{V},
\end{equation}
where $\mathbf{M} \in \mathbb{R}^{N_{BS} \times L \times L}$ serves as the additive bias matrix governing token visibility. Let $L_b$ denote the actual valid length of the $b$-th sample in the batch. The entry $M_{b,i,j}$ determines the compatibility between the $i$-th query and the $j$-th key in sample $b$:
\begin{equation}
    M_{b,i,j} = 
    \begin{cases} 
    0 & \text{if } j \le L_b \\
    -\infty & \text{if } j > L_b
    \end{cases}.
\end{equation}
Here, the condition $j > L_b$ identifies that the $j$-th key corresponds to a padding token. By setting $M_{b,i,j} = -\infty$, the corresponding attention weight approaches zero after the softmax operation. This ensures that the value vector $\mathbf{v}_j$ at any padding position ($j > L_b$) makes no contribution to the aggregated context vector, thereby preserving feature purity.

% \subsubsection{Asymmetric Masking Strategy}
Note that the attention mask targets only the Keys ($\mathbf{K}$) while leaving Queries ($\mathbf{Q}$) unmasked. This asymmetric approach is grounded in the row-column independence of the attention matrix and is justified as follows:

\begin{itemize}
    \item \textbf{Masking Keys:} We strictly mask columns where $j > L_b$ to prevent valid tokens from attending to uninformative zero-padding. Without this constraint, the softmax probability mass would be diluted by padding tokens, introducing noise into valid CSI representations.
    
    \item \textbf{Unmasked Queries:} Queries at padding positions (\(i > L_b\)) are intentionally left unmasked. Due to the row-wise independence of attention computation, valid token representations remain unaffected. Since the corresponding outputs are excluded from the loss, gradient propagation remains stable, yielding a simpler and numerically robust attention formulation.
\end{itemize}

\subsection{Implementation Design}

\subsubsection{Network Architecture}
The proposed network architecture, as illustrated in Fig.~\ref{fig:framework}, employs a Vision Transformer (ViT) as the backbone for channel modeling~\cite{dosovitskiy2020image}. The input, represented as 3D feature tensors is first partitioned into patches and subjected to padding and masking via a MAE preprocessing stage to simulate incomplete observations. These processed tokens are then fed into a stack of encoder blocks, each comprising multi-head self-attention (MHSA) with attention masks, layer normalization, and multi-layer perceptron (MLP) sublayers, interconnected by residual connections to facilitate gradient flow. The decoder module leverages skip connections and attention masks to reconstruct the original signal dimensions, preserving spatial relationships through a transformer-based framework.

\subsubsection{CSI Prediction Task}
We consider three categories of channel prediction tasks: CSI reconstruction, time-domain prediction, and frequency-domain prediction, all of which are formulated under a unified general reconstruction framework.

\textbf{1. CSI Reconstruction.}
The general problem is to recover the complete CSI tensor from a partially observed subset:
\begin{equation}
\mathbf{H} = \Phi_{\text{rec}}\big(\mathbf{H}[\Omega]\big),
\end{equation}
where $\Omega$ denotes an arbitrary set of indices across the space--time--frequency (STF) dimensions, and $\Phi_{\text{rec}}(\cdot)$ is the reconstruction function.  

\textbf{2. Time-Domain Channel Prediction.}
A special case of reconstruction where $\Omega$ spans all subcarriers and antennas but only the first $T_h$ time blocks. The objective is to predict future CSI based on historical observations:
\begin{equation}
\mathbf{H}[T_h+1:T,:,:] = \Phi_{t}\big(\mathbf{H}[1:T_h,:,:]\big).
\end{equation}

\textbf{3. Frequency-Domain Channel Prediction}
Another special case of reconstruction where $\Omega$ spans all time blocks and antennas but only the first $K_u$ subcarriers. The goal is to extrapolate CSI to adjacent frequency bands:
\begin{equation}
\mathbf{H}[:,K_u+1:K,:] = \Phi_{f}\big(\mathbf{H}[:,1:K_u,:]\big).
\end{equation}

In summary, both time-domain and frequency-domain channel prediction can be regarded as special cases of the general CSI reconstruction problem.

\section{Experiments}
This section validates the effectiveness of our channel-adaptive pretraining framework in reconciling computational efficiency with robust generalization. In the following, we introduce our experimental settings and analyze the zero-shot performance, scalability, and impact of bucket granularity across diverse heterogeneous wireless scenarios.
\subsection{Experiment Settings}

We conduct the experimental evaluations on two computing platforms. One workstation is equipped with an Intel(R) Xeon(R) Gold 5320 CPU @ 2.20 GHz, 256 GB RAM, and four NVIDIA GeForce RTX 4090 GPUs. The other server-grade platform is configured with an Intel(R) Xeon(R) Platinum 8468V CPU @ 3.80 GHz, 1.5 TB RAM, and eight NVIDIA H800 GPUs. The implementation and execution of model training are carried out using the PyTorch framework~\cite{PaszkeGMLBCKLGA19}.

\subsubsection{Dataset}
We construct CSI datasets encompassing diverse space-time-frequency configurations, and user mobility patterns. These datasets are generated using the 3GPP-standard compliant channel generator QuaDRiGa~\cite{jaeckel2014quadriga}, adopting a MISO-OFDM system architecture where the base station is equipped with a UPA planar array antenna and user equipment employs a single antenna configuration with half-wavelength spacing at the central frequency.

During the training phase, 40 datasets are utilized, each containing 12,000 samples randomly partitioned into training, validation, and test sets with a 9,000:1,000:2,000 ratio. For evaluation, 12 additional datasets with configurations distinct from the pre-training phase are generated, as summarized in Table II. These datasets serve as zero-shot test scenarios for our framework while acting as training data for the full-shot baselines. Complex Gaussian noise at 20 dB is added to all CSI samples during both training and inference processes. Detailed configuration parameters are available in the source code.

\subsubsection{Baseline}

To verify the superiority of our framework, we conduct a comparative study from two distinct perspectives: \textbf{model architecture} and \textbf{training paradigm}. For a fair comparison, all deep-learning-based baselines are implemented with the same decentralized training algorithm, identical learning rates, and consistent optimization settings. The specific baselines are detailed as follows:

\begin{itemize}
  \item LLM4CP~\cite{liu2024llm4cp}: LLM4CP is a LLM-empowered channel prediction scheme, where GPT-2 is finetuned for cross-modality knowledge transfer.
  \item WiFo~\cite{liu2025wifo}: The state-of-the-art wireless foundation model supporting multi-scale inputs. However, it employs a scale- and scenario-isolated training strategy without cross-scale shuffling, thereby failing to address the optimization challenges posed by dual heterogeneity.
  \item Transformer~\cite{jiang2022accurate}: A conventional sequence-to-sequence transformer that models CSI as flattened 1D time-series data, fundamentally differing from our ViT backbone which utilizes 3D patching to capture high-dimensional spatial-temporal structures.
  \item LSTM~\cite{jiang2020deep}: A two-layer LSTM network adopted as the sequential baseline. To accommodate 3D CSI tensors, we flatten the non-predictive dimensions—specifically merging antenna and frequency for time-domain tasks, or antenna and time for frequency-domain tasks—to construct the feature vectors.
  \item PAD~\cite{yin2020addressing}: A Prony-based estimator operating in the angular-delay domain, which is inherently restricted to time-domain extrapolation tasks. For implementation, we configure the predictor order as $N = 4$ and $N = 6$ corresponding to observation windows of $T = 16$ and $T = 24$, respectively.

\end{itemize}

\begin{table*}[t]
    \caption{Performance Comparison on Datasets D1--D12 in NMSE (dB). 'Rec', 'Pre-T', and 'Pre-F' denote CSI Reconstruction, Time-domain Prediction, and Frequency-domain Prediction tasks. The best performance in each row is highlighted in bold. The 'Improvement' column indicates the gain of the Proposed method over the best baseline (including both Full-shot and Zero-shot methods).}
    \label{tab:performance_comparison}
    \centering
    \resizebox{\textwidth}{!}{%
    \begin{tabular}{l l c l c c c c c c c c}
    \toprule
    \multirow{2}{*}{\textbf{Dataset}} & \multirow{2}{*}{\textbf{Scenario}} & \textbf{Freq.} & \multirow{2}{*}{\textbf{Task}} & \multicolumn{4}{c}{\textbf{Full-shot}} & \multicolumn{3}{c}{\textbf{Zero-shot}} & \textbf{Improvement} \\
    \cmidrule(lr){5-8} \cmidrule(lr){9-11} \cmidrule(lr){12-12}
     &  & (GHz) &  & \textbf{LSTM} & \textbf{Transformer} & \textbf{LLM4CP} & \textbf{BERT4MIMO} & \textbf{PAD} & \textbf{WiFo} & \textbf{Proposed} & \textbf{vs. Best} \\
    \midrule
    
    % --- D1 (Input D5) ---
    % Best Base: BERT(-18.15), Prop -19.26 -> +1.11
    % Best Base: LSTM(-19.50), Prop -19.42 -> -0.08
    % Best Base: LLM(-18.03), Prop -18.43 -> +0.40
    \multirow{3}{*}{\textbf{D1}} & \multirow{3}{*}{Indoor} & \multirow{3}{*}{2.4}
     & Rec   & - & - & - & -18.15 & - & -7.29 & \textbf{-19.26} & +1.11 \\
     & & & Pre-T & \textbf{-19.50} & -18.91 & -19.34 & - & -6.16 & -6.24 & -19.42 & -0.08 \\
     & & & Pre-F & -16.85 & -17.57 & -18.03 & - & - & -5.46 & \textbf{-18.43} & +0.40 \\
    \midrule
    
    % --- D2 (Input D14) ---
    % Best Base: BERT(-12.81), Prop -15.91 -> +3.10
    % Best Base: LLM(-18.58), Prop -18.50 -> -0.08
    % Best Base: LSTM(-5.57), Prop -8.95 -> +3.38
    \multirow{3}{*}{\textbf{D2}} & \multirow{3}{*}{Indoor} & \multirow{3}{*}{39.0}
     & Rec   & - & - & - & -12.81 & - & -5.95 & \textbf{-15.91} & +3.10 \\
     & & & Pre-T & -17.50 & -17.67 & \textbf{-18.58} & - & -1.03 & -4.61 & -18.50 & -0.08 \\
     & & & Pre-F & -5.57 & -5.54 & -4.88 & - & - & -1.65 & \textbf{-8.95} & +3.38 \\
    \midrule
    
    % --- D3 (Input D4) ---
    % Best Base: BERT(-17.16), Prop -17.69 -> +0.53
    % Best Base: LLM(-14.54), Prop -15.52 -> +0.98
    % Best Base: LLM(-5.85), Prop -9.03 -> +3.18
    \multirow{3}{*}{\textbf{D3}} & \multirow{3}{*}{Indoor} & \multirow{3}{*}{42.0}
     & Rec   & - & - & - & -17.16 & - & -6.77 & \textbf{-17.69} & +0.53 \\
     & & & Pre-T & -10.93 & -12.48 & -14.54 & - & 3.12 & -3.02 & \textbf{-15.52} & +0.98 \\
     & & & Pre-F & -5.43 & -5.47 & -5.85 & - & - & -1.08 & \textbf{-9.03} & +3.18 \\
    \midrule
    
    % --- D4 (Input D21) ---
    % Best Base: BERT(-18.27), Prop -18.04 -> -0.23
    % Best Base: LLM(-6.28), Prop -9.30 -> +3.02
    % Best Base: LLM(-12.47), Prop -16.87 -> +4.40
    \multirow{3}{*}{\textbf{D4}} & \multirow{3}{*}{RMa} & \multirow{3}{*}{2.6}
     & Rec   & - & - & - & \textbf{-18.27} & - & -11.82 & -18.04 & -0.23 \\
     & & & Pre-T & -5.07 & -5.64 & -6.28 & - & 3.76 & -4.64 & \textbf{-9.30} & +3.02 \\
     & & & Pre-F & -9.39 & -11.06 & -12.47 & - & - & -10.00 & \textbf{-16.87} & +4.40 \\
    \midrule
    
    % --- D5 (Input D16) ---
    % Best Base: BERT(-13.84), Prop -16.72 -> +2.88
    % Best Base: LLM(-13.30), Prop -15.52 -> +2.22
    % Best Base: LLM(-4.31), Prop -8.47 -> +4.16
    \multirow{3}{*}{\textbf{D5}} & \multirow{3}{*}{RMa} & \multirow{3}{*}{4.8}
     & Rec   & - & - & - & -13.84 & - & -9.07 & \textbf{-16.72} & +2.88 \\
     & & & Pre-T & -7.86 & -9.58 & -13.30 & - & 3.15 & -6.95 & \textbf{-15.52} & +2.22 \\
     & & & Pre-F & -3.51 & -3.91 & -4.31 & - & - & -3.31 & \textbf{-8.47} & +4.16 \\
    \midrule

    % --- D6 (Input D3) ---
    % Best Base: BERT(-16.70), Prop -15.29 -> -1.41
    % Best Base: WiFo(-2.87) vs LLM(-2.52) -> WiFo Best! Prop -4.42 -> +1.55 (Updated)
    % Best Base: LLM(-6.67), Prop -10.81 -> +4.14
    \multirow{3}{*}{\textbf{D6}} & \multirow{3}{*}{RMa} & \multirow{3}{*}{37.5}
     & Rec   & - & - & - & \textbf{-16.70} & - & -10.47 & -15.29 & -1.41 \\
     & & & Pre-T & -1.13 & -1.45 & -2.52 & - & 4.71 & -2.87 & \textbf{-4.42} & +1.55 \\
     & & & Pre-F & -3.51 & -5.21 & -6.67 & - & - & -6.55 & \textbf{-10.81} & +4.14 \\
    \midrule
    
    % --- D7 (Input D20) ---
    % Best Base: BERT(-10.78), Prop -15.55 -> +4.77
    % Best Base: LLM(-11.91), Prop -14.63 -> +2.72
    % Best Base: Trans(-10.59), Prop -15.76 -> +5.17
    \multirow{3}{*}{\textbf{D7}} & \multirow{3}{*}{UMa} & \multirow{3}{*}{2.1}
     & Rec   & - & - & - & -10.78 & - & -6.62 & \textbf{-15.55} & +4.77 \\
     & & & Pre-T & -10.36 & -11.22 & -11.91 & - & 3.35 & -5.96 & \textbf{-14.63} & +2.72 \\
     & & & Pre-F & -9.24 & -10.59 & -10.31 & - & - & -6.67 & \textbf{-15.76} & +5.17 \\
    \midrule
    
    % --- D8 (Input D11) ---
    % Best Base: BERT(-17.47), Prop -15.67 -> -1.80
    % Best Base: LLM(-4.91), Prop -6.48 -> +1.57
    % Best Base: LLM(-11.21), Prop -14.76 -> +3.55
    \multirow{3}{*}{\textbf{D8}} & \multirow{3}{*}{UMa} & \multirow{3}{*}{26.5}
     & Rec   & - & - & - & \textbf{-17.47} & - & -6.63 & -15.67 & -1.80 \\
     & & & Pre-T & -4.37 & -4.53 & -4.91 & - & 4.03 & -3.39 & \textbf{-6.48} & +1.57 \\
     & & & Pre-F & -8.16 & -11.16 & -11.21 & - & - & -4.12 & \textbf{-14.76} & +3.55 \\
    \midrule

    % --- D9 (Input D22) ---
    % Best Base: BERT(-10.59), Prop -13.01 -> +2.42
    % Best Base: LLM(-5.64), Prop -7.65 -> +2.01
    % Best Base: LLM(-10.92), Prop -15.22 -> +4.30
    \multirow{3}{*}{\textbf{D9}} & \multirow{3}{*}{UMa} & \multirow{3}{*}{42.0}
     & Rec   & - & - & - & -10.59 & - & -7.00 & \textbf{-13.01} & +2.42 \\
     & & & Pre-T & -5.64 & -5.58 & -5.64 & - & 3.28 & -5.07 & \textbf{-7.65} & +2.01 \\
     & & & Pre-F & -8.34 & -10.44 & -10.92 & - & - & -9.92 & \textbf{-15.22} & +4.30 \\
    \midrule

    % --- D10 (Input D18) ---
    % Best Base: BERT(-9.50), Prop -11.14 -> +1.64
    % Best Base: LLM(-9.05), Prop -9.80 -> +0.75
    % Best Base: WiFo(-2.13) vs Trans(-1.30) -> WiFo Best! Prop -3.63 -> +1.50 (Updated)
    \multirow{3}{*}{\textbf{D10}} & \multirow{3}{*}{UMi} & \multirow{3}{*}{2.1}
     & Rec   & - & - & - & -9.50 & - & -7.95 & \textbf{-11.14} & +1.64 \\
     & & & Pre-T & -4.90 & -6.31 & -9.05 & - & 3.54 & -8.21 & \textbf{-9.80} & +0.75 \\
     & & & Pre-F & -0.99 & -1.30 & -1.11 & - & - & -2.13 & \textbf{-3.63} & +1.50 \\
    \midrule

    % --- D11 (Input D23) ---
    % Best Base: BERT(-11.79), Prop -10.91 -> -0.88
    % Best Base: WiFo(-4.10) vs LLM(-3.88) -> WiFo Best! Prop -4.47 -> +0.37 (Updated)
    % Best Base: WiFo(-3.34) vs LLM(-2.72) -> WiFo Best! Prop -4.80 -> +1.46 (Updated)
    \multirow{3}{*}{\textbf{D11}} & \multirow{3}{*}{UMi} & \multirow{3}{*}{6.2}
     & Rec   & - & - & - & \textbf{-11.79} & - & -7.57 & -10.91 & -0.88 \\
     & & & Pre-T & -1.21 & -2.30 & -3.88 & - & 3.57 & -4.10 & \textbf{-4.47} & +0.37 \\
     & & & Pre-F & -1.54 & -2.22 & -2.72 & - & - & -3.34 & \textbf{-4.80} & +1.46 \\
    \midrule

    % --- D12 (Input D12) ---
    % Best Base: BERT(-18.19), Prop -17.28 -> -0.91
    % Best Base: LLM(-6.14), Prop -7.68 -> +1.54
    % Best Base: LLM(-11.13), Prop -15.47 -> +4.34
    \multirow{3}{*}{\textbf{D12}} & \multirow{3}{*}{UMi} & \multirow{3}{*}{29.0}
     & Rec   & - & - & - & \textbf{-18.19} & - & -6.84 & -17.28 & -0.91 \\
     & & & Pre-T & -4.89 & -5.51 & -6.14 & - & 4.04 & -3.46 & \textbf{-7.68} & +1.54 \\
     & & & Pre-F & -8.58 & -10.20 & -11.13 & - & - & -4.21 & \textbf{-15.47} & +4.34 \\
    \midrule
    
    % --- Average ---
    \multirow{3}{*}{\textbf{Avg}} & \multirow{3}{*}{-} & \multirow{3}{*}{-}
     & Rec   & - & - & - & -13.41 & - & -7.55 & \textbf{-14.74} & +1.33 \\
     & & & Pre-T & -5.36 & -6.03 & -7.10 & - & 2.89 & -4.60 & \textbf{-8.68} & +1.58 \\
     & & & Pre-F & -5.17 & -5.94 & -6.20 & - & - & -4.07 & \textbf{-9.34} & +3.14 \\

    \bottomrule
    \end{tabular}%
    }
\end{table*}

\noindent Apart from architectural comparisons, we further evaluate our framework under different training paradigms to demonstrate its robustness in handling heterogeneous CSI:

\begin{itemize}
  \item Sequential~\cite{goodfellow2013empirical}: The model is trained on each dataset one after another in a fixed order, typically prone to catastrophic forgetting.
  \item Alternating~\cite{stickland2019bert}: The model is updated on batches from each dataset in a fixed sequence, excluding any cross-dataset sample shuffling. This approach is similar to the pretraining method of WiFo~\cite{liu2025wifo}.
  \item Global~\cite{bottou2010large}: The standard training paradigm where all samples are uniformly shuffled across the entire global dataset, serving as a performance benchmark.
\end{itemize}
\noindent For brevity, we use “Proposed” to represent our approach.

\subsubsection{Performance Metrics}
Normalized Mean Squared Error (NMSE) is a widely used metric in channel prediction and signal processing to quantify the discrepancy between the predicted values and the ground truth~\cite{wen2018deep}. It is mathematically defined as:
\begin{equation}
\text{NMSE} = \frac{\| \mathbf{H} - \hat{\mathbf{H}} \|_F^2}{\| \mathbf{H} \|_F^2},
\end{equation}
where $\mathbf{H}$ denotes the true CSI tensor, $\hat{\mathbf{H}}$ represents the predicted CSI, and $\| \cdot \|_F$ indicates the Frobenius norm.
In our experiments, we report the NMSE in decibels (dB) to better evaluate the reconstruction quality, calculated as:
\begin{equation}
\text{NMSE (dB)} = 10 \log_{10} (\text{NMSE}).
\end{equation}

\begin{figure*}[!t]
    \centering
    \includegraphics[width=7in]{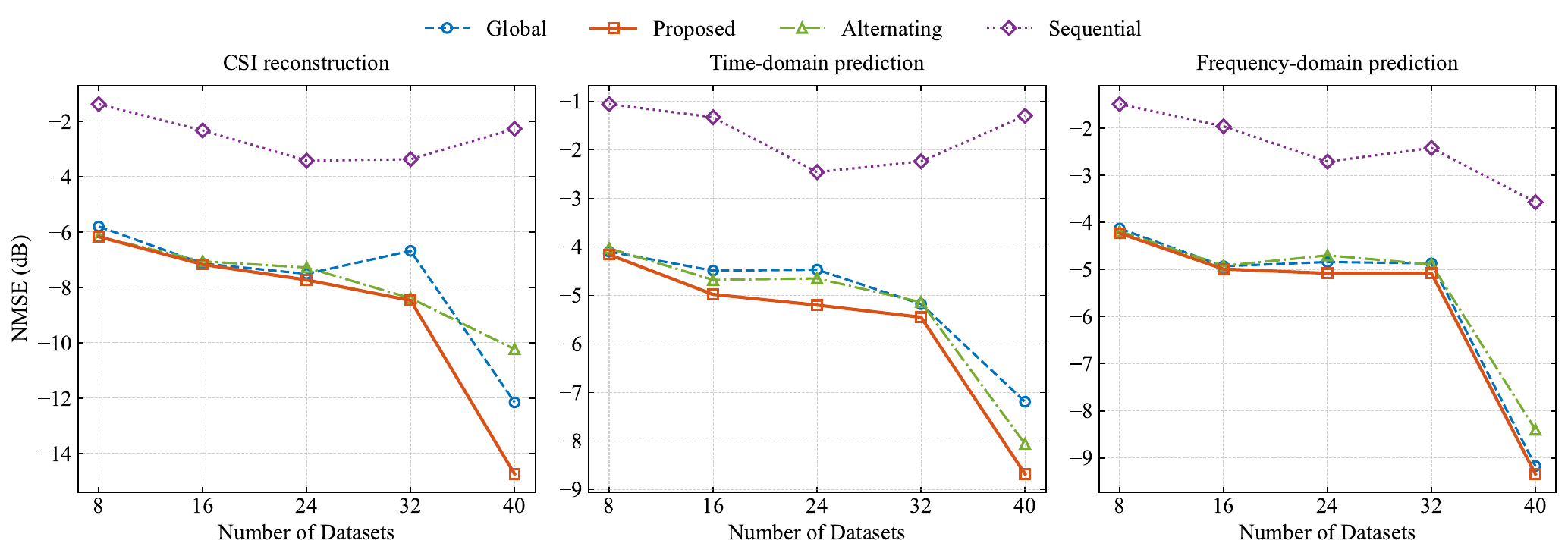}
    \caption{Zero-shot NMSE performance versus the number of training datasets for CSI reconstruction (left), time-domain prediction (middle), and frequency-domain prediction (right). The proposed method consistently outperforms baselines as data volume increases.}
    \label{fig:scalability}
    \end{figure*}

    \begin{table*}[t]
        \centering
        \caption{Efficiency Analysis: Padding Ratio vs. Training Time with Increasing Number of Datasets}
        \label{tab:efficiency_scaling}
        
        % 调整行高
        \renewcommand{\arraystretch}{1.25}
        
        % 设置表格宽度
        \begin{tabular*}{\textwidth}{@{\extracolsep{\fill}}lc cc cc cc cc cc}
        \toprule
        % 第一行表头
        \multirow{2.5}{*}{\textbf{Method}} & \multirow{2.5}{*}{\textbf{Bucket}} & \multicolumn{2}{c}{\textbf{8 Datasets}} & \multicolumn{2}{c}{\textbf{16 Datasets}} & \multicolumn{2}{c}{\textbf{24 Datasets}} & \multicolumn{2}{c}{\textbf{32 Datasets}} & \multicolumn{2}{c}{\textbf{40 Datasets}} \\
        
        % 分组横线
        \cmidrule(lr){3-4} \cmidrule(lr){5-6} \cmidrule(lr){7-8} \cmidrule(lr){9-10} \cmidrule(lr){11-12}
        
        % 第二行表头
         & & \textbf{Pad (\%)} & \textbf{Time} & \textbf{Pad (\%)} & \textbf{Time} & \textbf{Pad (\%)} & \textbf{Time} & \textbf{Pad (\%)} & \textbf{Time} & \textbf{Pad (\%)} & \textbf{Time} \\
        \midrule
        
        Sequential~\cite{goodfellow2013empirical} & - & 0 & 2.03 & 0 & 4.11 & 0 & 4.88 & 0 & 6.74 & 0 & 9.6 \\
        Alternating~\cite{stickland2019bert} & - & 0 & 2.03 & 0 & 4.11 & 0 & 4.88 & 0 & 6.74 & 0 & 9.6 \\
        Global~\cite{bottou2010large} & - & 36.46 & 2.54 & 32.81 & 5.08 & 28.82 & 6.66 & 44.06 & 13.11 & 58.91 & 30.55 \\

        \midrule
        
        \multirow{2}{*}{Proposed} 
         & 4 & 4.69 & 2.05 & 5.15 & 4.12 & 14.61 & 5.54 & 19.05 & 8.6 & 25.2 & 14.41 \\
         & 8 & 0 & 2.02 & 2.27 & 4.09 & 7.69 & 5.15 & 8.66 & 7.37 & 13.58 & 11.61 \\
         
        \bottomrule
        \end{tabular*}
    \end{table*}

\subsection{Zero-shot Generalization}

% Part 1: Dataset and Baseline Introduction
We conducted a comprehensive evaluation on 12 distinct zero-shot datasets (D1--D12) to assess the generalization capability of the proposed \textit{HeterCSI} framework. As detailed in Table~\ref{tab:performance_comparison}, these datasets span a wide spectrum of communication scenarios (Indoor, RMa, UMa, UMi), carrier frequencies (1.8 GHz to 42 GHz). The baselines are categorized into full-shot schemes (LSTM, Transformer, LLM4CP, BERT4MIMO), which are trained individually on target domains, and generalization-oriented approaches. The latter includes PAD (a training-free analytical method) and WiFo (a pretrained zero-shot model), both of which operate without access to the target domain's training data.

The results demonstrate that HeterCSI establishes a generalized foundation model without scenario-specific finetuning, exhibiting comprehensive superiority across all tasks. Compared to the existing zero-shot baseline WiFo, our framework significantly reduces the average NMSE by 7.19~dB, 4.08~dB, and 5.27~dB for CSI reconstruction, time-domain prediction, and frequency-domain prediction, respectively. More notably, even when compared against the best-performing baselines (including domain-specific Full-shot models), our method achieves average NMSE reductions of 1.33~dB, 1.58~dB, and 3.14~dB across the three tasks. This robust performance is attributed to the proposed scale-aware adaptive batching strategy, which enables the model to effectively mine and align universal representations from massive heterogeneous CSI.

Detailed statistical analysis reveals that our zero-shot model rivals or surpasses domain-specific experts. Specifically, \textit{HeterCSI} achieves the best performance in 29 out of 36 evaluation instances (approx. 80.6\%) against all baselines. Crucially, in the few cases where large-scale Full-shot models (e.g., BERT4MIMO) perform slightly better, our framework demonstrates exceptional stability, consistently securing the top-2 position (100\% rate) across the entire benchmark. Furthermore, the performance gap in these non-optimal cases is marginal and strictly controlled within 1.80~dB. These results confirm that \textit{HeterCSI} can serve as a reliable wireless foundation model, delivering optimal or near-optimal performance without the need for scenario-specific retraining.

% \begin{figure}[!t]
%     \centering
%     \includegraphics[width=3in,keepaspectratio]{figs/training_time_comparison.pdf}
%     \caption{Comparison of training time per epoch versus the number of datasets.}
%       \label{fig:overhead}
%     \end{figure}

\subsection{Scalability of Cross-Scale Pretraining}

Fig.~\ref{fig:scalability} illustrates the zero-shot NMSE performance trends as the training datasets expand from 8 to 40. A consistent scaling law is observed, where the generalization capability improves monotonically with data volume. However, the \textit{Global} strategy exhibits significant optimization instability, particularly in the CSI reconstruction task. This is evidenced by a distinct performance degradation (an anomalous rise in NMSE) at 32 datasets, attributed to destructive gradient conflicts arising from the naive mixing of heterogeneous scales. In contrast, the \textit{Proposed} framework maintains a robust generalization trajectory, effectively mitigating these conflicts. When using 40 datasets, our method establishes substantial performance advantages of approximately 2.8~dB and 1.6~dB in CSI reconstruction and time-domain prediction, respectively, while the improvement in frequency-domain prediction is 0.2 dB. The results indicate that the two core tasks, CSI reconstruction and time-domain prediction, are more susceptible to gradient conflicts induced by scale heterogeneity. Our approach effectively addresses this issue.

Furthermore, the proposed method demonstrates distinct advantages over the \textit{Sequential} and \textit{Alternating} baselines. While both approaches avoid padding-induced overhead, they fail to scale effectively with increasing data volume. The \textit{Sequential} strategy suffers from severe catastrophic forgetting, rendering it non-scalable with consistently high error levels across all tasks. The \textit{Alternating} approach fails to capture universal channel representations, causing it to stagnate at a significantly higher error floor compared to our method. Specifically, in the data-rich regime (40 datasets), the \textit{Proposed} framework consistently outperforms the \textit{Alternating} baseline, achieving NMSE reductions of 4.6dB, 0.7dB, and 0.9~dB for CSI reconstruction, time-domain prediction, and frequency-domain prediction, respectively. This comprehensive superiority underscores the necessity of our bucket-based shuffling, which ensures the model is exposed to global scenario diversity while maintaining local scale consistency.

Table~\ref{tab:efficiency_scaling} presents a quantitative analysis of padding ratios and training latency across increasing dataset volumes. Since padding tokens consume computational resources without contributing valid information, the padding ratio is directly correlated with training time, making our objective of minimizing padding tokens critical for system scalability. As shown in the table, the \textit{Global} strategy incurs a super-linear increase in computational costs; at 40 datasets, the padding ratio surges to 58.91\%, causing the training time to escalate to 30.55. In contrast, the \textit{Proposed} framework effectively suppresses this overhead through bucket-based shuffling, maintaining a near-linear growth trajectory. Specifically, with a bucket granularity of $B=8$, our method restricts the padding ratio to 13.58\% even at the largest scale, reducing training time by approximately 52.83\% compared to the \textit{Global} strategy (14.41 vs. 30.55). This demonstrates that our framework successfully achieves high computational efficiency comparable to the non-padded \textit{Alternating} method (9.6), while retaining the generalization benefits of cross-dataset shuffling.

% \begin{figure}[!t]
%     \centering
%     \subfloat[Low heterogeneity]{\includegraphics[width=1.7in]{figs/bucket_average_nmse_ieee.pdf}}%
%     \hfil
%     \subfloat[High heterogeneity]{\includegraphics[width=1.7in]{figs/bucket_average_nmse_ieee.pdf}}
%     \caption{Comparison of gradient cosine similarity matrices within a training batch.}
%     \label{fig:bucket_sensitivity}
% \end{figure}

\begin{figure}[!t]
    \centering
    \includegraphics[width=3in,keepaspectratio]{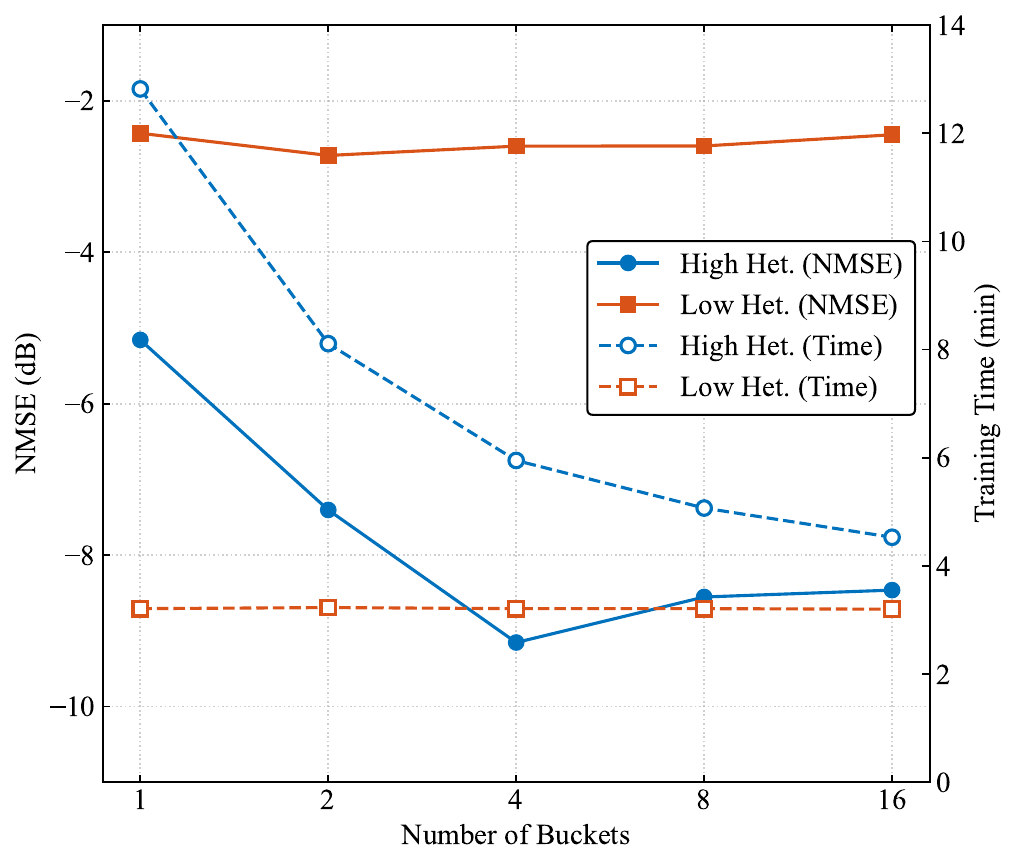}
    \caption{Impact of bucket granularity on zero-shot NMSE performance and training time under different heterogeneity levels.}
    \label{fig:heter}
\end{figure}

% \begin{figure}[!t]
%     \centering
%     \includegraphics[width=3in,keepaspectratio]{figs/bucket_average_nmse_ieee.pdf}
%     \caption{Comparison of training time per epoch versus the number of datasets.}
%     \label{fig:high_heter}
% \end{figure}

\subsection{Impact of Bucket Granularity and Heterogeneity}

Fig.~\ref{fig:heter} illustrates the impact of bucket granularity ($B$) on zero-shot NMSE performance under different levels of dataset heterogeneity. The model trained on highly heterogeneous data demonstrates a substantial performance advantage over the model trained on homogeneous data. For instance, at a bucket granularity of $B=4$, the high-heterogeneity setting achieves NMSE of approximately $-9.1$~dB, whereas the low-heterogeneity counterpart stagnates at $-2.6$~dB, yielding a significant performance gap of $6.5$~dB. This disparity provides empirical evidence that pre-training solely on datasets with limited diversity or similar scales fails to cultivate the robust feature representations necessary for zero-shot generalization to unseen scenarios. In contrast, leveraging large-scale heterogeneous CSI allows the model to learn universal channel properties that remain valid across diverse environments.

Furthermore, the impact of bucket resolution on training latency reveals a critical mechanism for resolving the \textit{efficiency-generalization dilemma}. As depicted by the dashed curves, the high-heterogeneity scenario suffers from excessive padding overhead at $B=1$, resulting in a training time of $12.5$~min. However, increasing $B$ to $4$ drastically reduces this overhead, cutting the training time by approximately $52\%$ to $6.0$~min, while the low-heterogeneity case maintains a consistent low latency ($\approx 3.3$~min). This observation, combined with the NMSE trends, highlights an optimal trade-off: high-heterogeneity scenarios benefit from a finer partition ($B=4$) to balance feature granularity and optimization efficiency, whereas low-heterogeneity data favors coarser quantization (optimal at $B=2$). The existence of these distinct optima confirms that the proposed framework successfully reconciles the conflict between minimizing padding overhead and maintaining gradient diversity. By utilizing $B$ as a tunable control knob, the framework effectively mitigates scale-induced gradient conflicts without sacrificing the stochasticity required to prevent overfitting, thereby ensuring robust generalization even in highly complex channel environments.

% \subsubsection{Zero-padding Overhead}

\section{Conclusion}
This paper presents HeterCSI, a channel-adaptive pretraining framework designed to reconcile the trade-off between optimization stability and generalization in wireless foundation models. By implementing a scale-aware adaptive batching strategy alongside a double masking mechanism, the proposed approach effectively mitigates destructive gradient interference caused by scale heterogeneity while preserving the constructive benefits of scenario diversity. Empirical validation across 12 unseen scenarios confirms the superior zero-shot performance of our framework, outperforming full-shot baselines on average. Specifically, compared to the state-of-the-art zero-shot benchmark WiFo, HeterCSI yields average NMSE reductions of 7.19~dB, 4.08~dB, and 5.27~dB for CSI reconstruction, time-domain prediction, and frequency-domain prediction tasks, respectively. In future work, we will continue to exploit the intrinsic properties of CSI, by investigating physics-aware positional encodings and specialized model architectures to enhance representation learning in dynamic wireless environments.

%%%%%%%%%%%%%%%%%%%%%%%%%%%%%%%%%%%%%%%%%%%%%%%%%%%%%%%
%%% Acknowledgements. ��л
%%%%%%%%%%%%%%%%%%%%%%%%%%%%%%%%%%%%%%%%%%%%%%%%%%%%%%%
% \Acknowledgements{This work was supported by the National Natural Science Foundation of China (Grant Nos. 00000000 and 11111111).}

%%%%%%%%%%%%%%%%%%%%%%%%%%%%%%%%%%%%%%%%%%%%%%%%%%%%%%%
%%% Supplements. �������, �Ǳ�ѡ
%%%%%%%%%%%%%%%%%%%%%%%%%%%%%%%%%%%%%%%%%%%%%%%%%%%%%%%

%%%%%%%%%%%%%%%%%%%%%%%%%%%%%%%%%%%%%%%%%%%%%%%%%%%%%%%
%%% Open Access Funding Note. ���Ż�ȡ������Դ˵��
%%%%%%%%%%%%%%%%%%%%%%%%%%%%%%%%%%%%%%%%%%%%%%%%%%%%%%%
%\FundingNote{\textbf{Open Access} funding enabled and organized by CAUL and its Member Institutions.}

%%%%%%%%%%%%%%%%%%%%%%%%%%%%%%%%%%%%%%%%%%%%%%%%%%%%%%%
%%% Reference section. �ο�����
%%% Citation in the content using "some words~\cite{1,2}".
%%% ~ is needed to make the reference number is on the same line with the word before it.
%%% Using scis.bst to format the style if the ref.bib file is included, e.g.,
\bibliographystyle{IEEEtran}
\bibliography{cite}

@article{gao2025enabling,
  title={Enabling {6G} Through Multi-Domain Channel Extrapolation: Opportunities and Challenges of Generative Artificial Intelligence},
  author={Gao, Yuan and Lu, Zichen and Wu, Yifan and Jin, Yanliang and Zhang, Shunqing and Chu, Xiaoli and Xu, Shugong and Wang, Cheng-Xiang},
  journal={arXiv preprint arXiv:2509.01125},
  year={2025}
}

@article{zhang2023ai,
  title={{AI-based} time-, frequency-, and space-domain channel extrapolation for {6G}: Opportunities and challenges},
  author={Zhang, Zhen and Zhang, Jianhua and Zhang, Yuxiang and Yu, Li and Liu, Guangyi},
  journal={IEEE Vehicular Technology Magazine},
  volume={18},
  number={1},
  pages={29--39},
  year={2023},
  publisher={IEEE}
}

@article{akrout2023domain,
  title={Domain generalization in machine learning models for wireless communications: Concepts, state-of-the-art, and open issues},
  author={Akrout, Mohamed and Feriani, Amal and Bellili, Faouzi and Mezghani, Amine and Hossain, Ekram},
  journal={IEEE Communications Surveys \& Tutorials},
  volume={25},
  number={4},
  pages={3014--3037},
  year={2023},
  publisher={IEEE}
}

@article{liu2025wifo,
  title={{WiFo}: Wireless foundation model for channel prediction},
  author={Liu, Boxun and Gao, Shijian and Liu, Xuanyu and Cheng, Xiang and Yang, Liuqing},
  journal={Science China Information Sciences},
  volume={68},
  number={6},
  pages={162302},
  year={2025},
  publisher={Springer}
}

@inproceedings{he2022masked,
  title={Masked autoencoders are scalable vision learners},
  author={He, Kaiming and Chen, Xinlei and Xie, Saining and Li, Yanghao and Doll{\'a}r, Piotr and Girshick, Ross},
  booktitle={Proceedings of the IEEE/CVF conference on computer vision and pattern recognition},
  pages={16000--16009},
  year={2022}
}

@article{yang2025wirelessgpt,
  title={{WirelessGPT}: A generative pre-trained multi-task learning framework for wireless communication},
  author={Yang, Tingting and Zhang, Ping and Zheng, Mengfan and Shi, Yuxuan and Jing, Liwen and Huang, Jianbo and Li, Nan},
  journal={IEEE Network},
  year={2025},
  publisher={IEEE}
}

@article{hong2024spectralgpt,
  title={{SpectralGPT}: Spectral Remote Sensing Foundation Model},
  author={Hong, Danfeng and Zhang, Bing and Li, Xuyang and Li, Yuxuan and Li, Chenyu and Yao, Jing and Yokoya, Naoto and Li, Hao and Ghamisi, Pedram and Jia, Xiuping and others},
  journal={IEEE Transactions on Pattern Analysis \& Machine Intelligence},
  volume={46},
  number={08},
  pages={5227--5244},
  year={2024},
  publisher={IEEE Computer Society}
}

@inproceedings{PaszkeGMLBCKLGA19,
  author    = {Adam Paszke and Sam Gross and Francisco Massa and Adam Lerer and others},
  title     = {PyTorch: An Imperative Style, High-Performance Deep Learning Library},
  booktitle = {Advances in Neural Information Processing Systems},
  volume    = {32},
  pages     = {8024--8035},
  year      = {2019}
}

@article{jiang2020deep,
  title={Deep learning for fading channel prediction},
  author={Jiang, Wei and Schotten, Hans Dieter},
  journal={IEEE Open Journal of the Communications Society},
  volume={1},
  pages={320--332},
  year={2020},
  publisher={IEEE}
}

@article{liu2024llm4cp,
  title={{LLM4CP}: Adapting large language models for channel prediction},
  author={Liu, Boxun and Liu, Xuanyu and Gao, Shijian and Cheng, Xiang and Yang, Liuqing},
  journal={Journal of Communications and Information Networks},
  volume={9},
  number={2},
  pages={113--125},
  year={2024},
  publisher={PTP}
}

@article{alikhani2024large,
  title={Large wireless model ({LWM}): A foundation model for wireless channels},
  author={Alikhani, Sadjad and Charan, Gouranga and Alkhateeb, Ahmed},
  journal={arXiv preprint arXiv:2411.08872},
  year={2024}
}

@article{liu2025llm4wm,
  title={{LLM4WM}: Adapting llm for wireless multi-tasking},
  author={Liu, Xuanyu and Gao, Shijian and Liu, Boxun and Cheng, Xiang and Yang, Liuqing},
  journal={IEEE Transactions on Machine Learning in Communications and Networking},
  year={2025},
  volume={3},
  number={},
  pages={835-847},
  publisher={IEEE}
}

@article{jaeckel2014quadriga,
  title={{QuaDRiGa}: A {3-D} multi-cell channel model with time evolution for enabling virtual field trials},
  author={Jaeckel, Stephan and Raschkowski, Leszek and B{\"o}rner, Kai and Thiele, Lars},
  journal={IEEE transactions on antennas and propagation},
  volume={62},
  number={6},
  pages={3242--3256},
  year={2014},
  publisher={IEEE}
}

@article{guo2025lvm4csi,
  title={{LVM4CSI}: Enabling Direct Application of Pre-Trained Large Vision Models for Wireless Channel Tasks},
  author={Guo, Jiajia and Jiang, Peiwen and Wen, Chao-Kai and Jin, Shi and Zhang, Jun},
  journal={arXiv preprint arXiv:2507.05121},
  year={2025}
}

@article{pan2025large,
  title={Large Wireless Localization Model {(LWLM)}: A Foundation Model for Positioning in {6G Networks}},
  author={Pan, Guangjin and Huang, Kaixuan and Chen, Hui and Zhang, Shunqing and H{\"a}ger, Christian and Wymeersch, Henk},
  journal={arXiv preprint arXiv:2505.10134},
  year={2025}
}

@article{dosovitskiy2020image,
  title={An image is worth 16x16 words: Transformers for image recognition at scale},
  author={Dosovitskiy, Alexey and Beyer, Lucas and Kolesnikov, Alexander and Weissenborn, Dirk and Zhai, Xiaohua and Unterthiner, Thomas and Dehghani, Mostafa and Minderer, Matthias and Heigold, Georg and Gelly, Sylvain and others},
  journal={arXiv preprint arXiv:2010.11929},
  year={2020}
}

@article{kim2020massive,
  title={Massive {MIMO} channel prediction: Kalman filtering vs. machine learning},
  author={Kim, Hwanjin and Kim, Sucheol and Lee, Hyeongtaek and Jang, Chulhee and Choi, Yongyun and Choi, Junil},
  journal={IEEE Transactions on Communications},
  year={2021},
  volume={69},
  number={1},
  pages={518-528},
  publisher={IEEE}
}

@article{li2025bridging,
  title={Bridging the Modality Gap: Enhancing Channel Prediction with Semantically Aligned LLMs and Knowledge Distillation},
  author={Li, Zhaoyang and Yang, Qianqian and Xiong, Zehui and Shi, Zhiguo and Quek, Tony QS},
  journal={arXiv preprint arXiv:2505.12729},
  year={2025}
}

@article{hussein2023least,
  title={Least Square Estimation-Based Different Fast Fading Channel Models in {MIMO-OFDM} Systems},
  author={Hussein, Walaa and Audah, Kamil and Noordin, NK and Kraiem, Habib and Flah, Aymen and Fadlee, Mohd and Ismail, Alyani},
  journal={International Transactions on Electrical Energy Systems},
  volume={2023},
  number={1},
  pages={5547634},
  year={2023},
  publisher={Wiley Online Library}
}

@article{bacci2024mmse,
  title={{MMSE} channel estimation in large-scale MIMO: Improved robustness with reduced complexity},
  author={Bacci, Giacomo and D{’A}mico, Antonio Alberto and Sanguinetti, Luca},
  journal={IEEE Transactions on Wireless Communications},
  year={2024},
  publisher={IEEE}
}

@article{catak2025bert4mimo,
  title={{BERT4MIMO}: A foundation model using bert architecture for massive mimo channel state information prediction},
  author={Catak, Ferhat Ozgur and Kuzlu, Murat and Cali, Umit},
  journal={arXiv preprint arXiv:2501.01802},
  year={2025}
}

@article{jiang2022accurate,
  title={Accurate channel prediction based on transformer: Making mobility negligible},
  author={Jiang, Hao and Cui, Mingyao and Ng, Derrick Wing Kwan and Dai, Linglong},
  journal={IEEE Journal on Selected Areas in Communications},
  volume={40},
  number={9},
  pages={2717--2732},
  year={2022},
  publisher={IEEE}
}

@article{yu2024channelgpt,
  title={{ChannelGPT}: A large model to generate digital twin channel for 6G environment intelligence},
  author={Yu, Li and Shi, Lianzheng and Zhang, Jianhua and Wang, Jialin and Zhang, Zhen and Zhang, Yuxiang and Liu, Guangyi},
  journal={arXiv preprint arXiv:2410.13379},
  year={2024}
}

@article{yin2020addressing,
  title={Addressing the curse of mobility in massive {MIMO} with prony-based angular-delay domain channel predictions},
  author={Yin, Haifan and Wang, Haiquan and Liu, Yingzhuang and Gesbert, David},
  journal={IEEE Journal on Selected Areas in Communications},
  volume={38},
  number={12},
  pages={2903--2917},
  year={2020},
  publisher={IEEE}
}

@techreport{samsung2025_6g_mimo,
  title        = {{6G MIMO: Toward Intelligent and Energy-Efficient Radio}},
  author       = {{Gary Xu, Younghan Nam, Seunghyun Lee and others}},
  year         = {2025},
  url          = {https://research.samsung.com/blog/6G-MIMO-Toward-Intelligent-and-Energy-Efficient-Radio},
}

@article{helmy2023lstm,
  title={{LSTM-GRU} model-based channel prediction for one-bit massive MIMO system},
  author={Helmy, Islam and Tarafder, Pulok and Choi, Wooyeol},
  journal={IEEE Transactions on Vehicular Technology},
  volume={72},
  number={8},
  pages={11053--11057},
  year={2023},
  publisher={IEEE}
}

@article{xia2019deep,
  title={A Deep Learning Framework for Optimization of {MISO} Downlink Beamforming}, 
  author={Xia, Wenchao and Zheng, Gan and Zhu, Yongxu and Zhang, Jun and Wang, Jiangzhou and Petropulu, Athina P},
  journal={IEEE Transactions on Communications},
  year={2020},
  volume={68},
  number={3},
  pages={1866-1880},
  publisher={IEEE}
}

@article{guler2025multi,
  title={A Multi-Task Foundation Model for Wireless Channel Representation Using Contrastive and Masked Autoencoder Learning},
  author={Guler, Berkay and Geraci, Giovanni and Jafarkhani, Hamid},
  journal={arXiv preprint arXiv:2505.09160},
  year={2025}
}

@techreport{itu_r_imt2030,
  title        = {{Framework and overall objectives of the future development of IMT-2030 and beyond}},
  author       = {{ITU-R WP 5D}},
  institution  = {{International Telecommunication Union - Radiocommunication Sector (ITU-R)}},
  type         = {Report},
  number       = {IMT-2030},
  year         = {2024},
  month        = {July}
}

@techreport{3gpp_ts38_101_1,
  title        = {{3GPP TS 38.101-1: NR; User Equipment (UE) Radio Transmission and Reception; Part 1: Range 1 Standalone}},
  author       = {{3GPP}},
  institution  = {{3rd Generation Partnership Project (3GPP)}},
  type         = {Technical Specification},
  number       = {TS 38.101-1},
  version      = {Rel-16},
  year         = {2020},
  month        = {December}
}

@inproceedings{brown2020language,
 author = {Tom Brown, Benjamin Mann, Nick Ryder and others},
 booktitle = {Advances in Neural Information Processing Systems},
 editor = {H. Larochelle and M. Ranzato and R. Hadsell and M.F. Balcan and H. Lin},
 pages = {1877--1901},
 publisher = {Curran Associates, Inc.},
 title = {Language Models are Few-Shot Learners},
 volume = {33},
 year = {2020}
}

@article{wang2023comprehensive,
  title={Comprehensive survey on hard example mining in deep learning},
  author={Wang, Xiaofeng and others},
  journal={Artif. Intell. Rev.},
  volume={56},
  number={8},
  pages={7613--7665},
  year={2023}
}

@book{dahlman20205g,
  title={5G NR: The next generation wireless access technology},
  author={Dahlman, Erik and Parkvall, Stefan and Skold, Johan},
  year={2020},
  publisher={Academic Press}
}

@inproceedings{vaswani2017attention,
 author = {Vaswani, Ashish and Shazeer, Noam and Parmar and others},
 booktitle = {Advances in Neural Information Processing Systems},
 pages = {},
 publisher = {Curran Associates, Inc.},
 title = {Attention is All you Need},
 volume = {30},
 year = {2017}
}

@inproceedings{devlin2018bert,
  title={BERT: Pre-training of Deep Bidirectional Transformers for Language Understanding},
  author={Devlin, Jacob and Chang, Ming-Wei and Lee, Kenton and Toutanova, Kristina},
  booktitle={Proceedings of the 2019 Conference of the North American Chapter of the Association for Computational Linguistics: Human Language Technologies, Volume 1 (Long and Short Papers)},
  pages={4171--4186},
  year={2019}
}

@article{bajwa2010compressed,
  title={Compressed channel sensing: A new approach to estimating sparse multipath channels},
  author={Bajwa, Waheed U and Haupt, Jarvis and Sayeed, Akbar M and Nowak, Robert},
  journal={Proceedings of the IEEE},
  volume={98},
  number={6},
  pages={1058--1076},
  year={2010},
  publisher={IEEE}
}

@techreport{3gpp38901,
  author      = "{3rd Generation Partnership Project (3GPP)}",
  title       = "{Study on channel model for frequencies from 0.5 to 100 GHz}",
  institution = "{3rd Generation Partnership Project (3GPP)}",
  number      = "{TR 38.901 V16.1.0}",
  year        = {2019},
  month       = {12},
  type        = {Technical Report}
}

@article{wen2018deep,
  title={Deep learning for massive MIMO CSI feedback},
  author={Wen, Chao-Kai and Shih, Wan-Ting and Jin, Shi},
  journal={IEEE Wireless Communications Letters},
  volume={7},
  number={5},
  pages={748--751},
  year={2018},
  publisher={IEEE}
}

@article{goodfellow2013empirical,
  title={An empirical investigation of catastrophic forgetting in gradient-based neural networks},
  author={Goodfellow, Ian J and Mirza, Mehdi and Xiao, Da and Courville, Aaron and Bengio, Yoshua},
  journal={arXiv preprint arXiv:1312.6211},
  year={2013}
}

@inproceedings{bottou2010large,
  title={Large-scale machine learning with stochastic gradient descent},
  author={Bottou, L{\'e}on},
  booktitle={Proceedings of COMPSTAT'2010: 19th International Conference on Computational StatisticsParis France, August 22-27, 2010 Keynote, Invited and Contributed Papers},
  pages={177--186},
  year={2010},
  organization={Springer}
}

@article{zhang2026wifo,
  title={{WiFo-M}: Plug-and-Play Multi-Modal Sensing via Foundation Model to Empower Wireless Communications},
  author={Zhang, Haotian and Gao, Shijian and Cheng, Xiang},
  journal={arXiv preprint arXiv:2601.09179},
  year={2026}
}

@article{sun2025llm4pg,
  title={{LLM4PG}: Adapting Large Language Model for Pathloss Map Generation via Synesthesia of Machines},
  author={Sun, Mingran and Bai, Lu and Cheng, Xiang and Wu, Jianjun},
  journal={arXiv preprint arXiv:2511.02423},
  year={2025}
}

@article{fan2025sharp,
  title={Sharp Minima Can Generalize: A Loss Landscape Perspective On Data},
  author={Fan, Raymond and Sandlund, Bryce and Ko, Lin Myat},
  journal={arXiv preprint arXiv:2511.04808},
  year={2025}
}

@inproceedings{shi2024conflict,
  title={Conflict-alleviated gradient descent for adaptive object detection},
  author={Shi, Wenxu and Zheng, Bochuan},
  booktitle={Proceedings of the Thirty-Third International Joint Conference on Artificial Intelligence},
  pages={1236--1244},
  year={2024}
}

@inproceedings{hardt2016train,
  title={Train faster, generalize better: Stability of stochastic gradient descent},
  author={Hardt, Moritz and Recht, Ben and Singer, Yoram},
  booktitle={International conference on machine learning},
  pages={1225--1234},
  year={2016},
  organization={PMLR}
}

@article{keskar2016large,
  title={On large-batch training for deep learning: Generalization gap and sharp minima},
  author={Keskar, Nitish Shirish and Mudigere, Dheevatsa and Nocedal, Jorge and Smelyanskiy, Mikhail and Tang, Ping Tak Peter},
  journal={arXiv preprint arXiv:1609.04836},
  year={2016}
}

@article{balas1976set,
  title={Set partitioning: A survey},
  author={Balas, Egon and Padberg, Manfred W},
  journal={SIAM review},
  volume={18},
  number={4},
  pages={710--760},
  year={1976},
  publisher={SIAM}
}

@inproceedings{li2024quantity,
  title={From quantity to quality: Boosting llm performance with self-guided data selection for instruction tuning},
  author={Li, Ming and Zhang, Yong and Li, Zhitao and Chen, Jiuhai and Chen, Lichang and Cheng, Ning and Wang, Jianzong and Zhou, Tianyi and Xiao, Jing},
  booktitle={Proceedings of the 2024 Conference of the North American Chapter of the Association for Computational Linguistics: Human Language Technologies (Volume 1: Long Papers)},
  pages={7602--7635},
  year={2024}
}

@inproceedings{stickland2019bert,
  title={Bert and pals: Projected attention layers for efficient adaptation in multi-task learning},
  author={Stickland, Asa Cooper and Murray, Iain},
  booktitle={International Conference on Machine Learning},
  pages={5986--5995},
  year={2019},
  organization={PMLR}
}
%%%%%%%%%%%%%%%%%%%%%%%%%%%%%%%%%%%%%%%%%%%%%%%%%%%%%%%
% \begin{thebibliography}{99}

% \bibitem{1} Author A, Author B, Author C. Reference title. Journal, 2024, 38: 13--28

% \bibitem{2} Author A, Author B, Author C, et al. Reference title. In: Proceedings of Conference, Place, 2024. 6--12

% \end{thebibliography}

%%%%%%%%%%%%%%%%%%%%%%%%%%%%%%%%%%%%%%%%%%%%%%%%%%%%%%%
%%% Appendix sections. ��¼�½�, �Ǳ�ѡ
%%%%%%%%%%%%%%%%%%%%%%%%%%%%%%%%%%%%%%%%%%%%%%%%%%%%%%%
%\begin{appendix}
%\section{Name}

%\end{appendix}

\end{document}